\def\forarxiv{0}
\documentclass[twoside]{article}

\usepackage[accepted]{aistats2021}
\usepackage[round]{natbib}

\usepackage[switch,columnwise]{lineno}

\usepackage{shortcuts}
\usepackage{amsthm,amsmath,amssymb,microtype,graphicx,subcaption,booktabs,enumitem,hyperref,rotating,caption,calc}
\usepackage{tikz}
\usepackage[ruled,vlined]{algorithm2e}
\usepackage[normalem]{ulem}
\usepackage{authblk}

\usetikzlibrary{shapes,decorations,arrows,calc,arrows.meta,fit,positioning}
\tikzset{
    -Latex,auto,node distance =1 cm and 1 cm,semithick,
    state/.style ={ellipse, draw, minimum width = 0.7 cm},
    point/.style = {circle, draw, inner sep=0.04cm,fill,node contents={}},
    square/.style = {rectangle, draw, inner sep=0.04cm,fill,node contents={}},
    bidirected/.style={Latex-Latex,dashed},
    el/.style = {inner sep=2pt, align=left, sloped}
}

\theoremstyle{plain}
\newtheorem{theorem}{Theorem}
\newtheorem{proposition}[theorem]{Proposition}
\newtheorem{lemma}[theorem]{Lemma}

\theoremstyle{definition}
\newtheorem{assumption}{Assumption}

\newtheorem{remark}{Remark}

\DeclareCaptionStyle{figstyle}
  [format=plain,margin=0pt,justification=centering]
  {format=plain,margin=8pt}
\newcount\Comments  
\definecolor{darkgreen}{rgb}{0,0.5,0}
\definecolor{darkred}{rgb}{0.7,0,0}
\definecolor{teal}{rgb}{0.3,0.8,0.8}
\definecolor{blue}{rgb}{0,0,1}
\definecolor{purple}{rgb}{0.5,0,1}
\newcommand{\kibitz}[2]{\ifnum\Comments=1\textcolor{#1}{#2}\fi}

\begin{document}

%

%
\runningauthor{Adarsh Subbaswamy, Roy Adams, Suchi Saria}

\twocolumn[

\aistatstitle{Evaluating Model Robustness and Stability to Dataset Shift}

\aistatsauthor{ Adarsh Subbaswamy$^{*}$, Roy Adams$^*$  \And  Suchi Saria }

\aistatsaddress{Johns Hopkins University \And Johns Hopkins University \& Bayesian Health} ]


\begin{abstract}
    As the use of machine learning in high impact domains becomes widespread, the importance of evaluating safety has increased. An important aspect of this is evaluating how robust a model is to changes in setting or population, which typically requires applying the model to multiple, independent datasets. Since the cost of collecting such datasets is often prohibitive, in this paper, we propose a framework for analyzing this type of stability using the available data. We use the original evaluation data to determine distributions under which the algorithm performs poorly, and estimate the algorithm's performance on the ``worst-case'' distribution. We consider shifts in user defined conditional distributions, allowing some distributions to shift while keeping other portions of the data distribution fixed. For example, in a healthcare context, this allows us to consider shifts in clinical practice while keeping the patient population fixed. To address the challenges associated with estimation in complex, high-dimensional distributions, we derive a ``debiased'' estimator which maintains $\sqrt{N}$-consistency even when machine learning methods with slower convergence rates are used to estimate the nuisance parameters. In experiments on a real medical risk prediction task, we show this estimator can be used to analyze stability and accounts for realistic shifts that could not previously be expressed. The proposed framework allows practitioners to proactively evaluate the safety of their models without requiring additional data collection.
\end{abstract}

 \section{Introduction}
A high quality evaluation of a machine learning (ML) model must allow a user to determine if the model will help safely and effectively achieve their goals in the deployment environment. Such an evaluation should demonstrate whether a model will perform well in the deployment environment, which often differs from the environment in which training data was gathered (i.e., the case of \emph{dataset shift} \citep{quionero2009dataset}). Additionally, it should test whether the model will perform well across all relevant subpopulations and whether performance will deteriorate in unexpected ways as the deployment environment evolves over time. Historically, model evaluation methods have focused on performance on new or heldout data that ``looks similar’’ to the training data. Existing tools, such as cross-validation and the bootstrap, satisfy the assumption that the training and test data are drawn from the same distribution, but fail to capture potential differences between training and deployment environments. 
 
To capture such differences, a common practice is to evaluate performance on multiple, independently collected datasets. While this helps to address differences that can exist across environments, this approach has limitations. If we cannot completely characterize how the datasets differ, or if the datasets are not sufficiently diverse, then a user can draw only limited conclusions about how the model will perform in the deployment environment. To counteract this, one could consider targeted collection of additional datasets which differ in structured ways (e.g., collecting data with differing demographics). This approach can be prohibitively costly, or in some cases, impossible. For example, we can not ethically collect new loan approval datasets in which we forcibly vary customer spending habits. Thus, to fully capture the performance of a model in real-world environments, we need methods for analyzing the \emph{stability} of a model. Such stability analyses should demonstrate the range of environments and subpopulations in which a model performs well and which types of changes in environment will degrade performance.
 
Unfortunately, we lack methodology for performing stability analysis, which is increasingly needed as ML systems are being deployed across a number of industries, such as health care and personal finance, in which system performance translates directly to real-world outcomes (see, e.g., \citet{subbaswamy2020development} for a discussion of the need for shift-stable models in healthcare). Further, as regulations and product reviews become more common across industries (e.g., the United States Food and Drug Administration (FDA) currently regulates ML systems for medical applications), system developers will be expected to produce high quality evaluations to prove the safety of their systems \citep{food2019proposed,food2020artificial}. For example, consider evaluating a model trained to diagnose a disease $Y$ from a set of covariates $X$ (e.g., age, medical history, treatments). We may wish to learn what differences across hospitals would lead to model failures, such as changes to the patient demographics or differences in clinical treatment practices.

\begin{figure}[t]
\vspace{-0.1in}
\centering
 \includegraphics[scale=0.33]{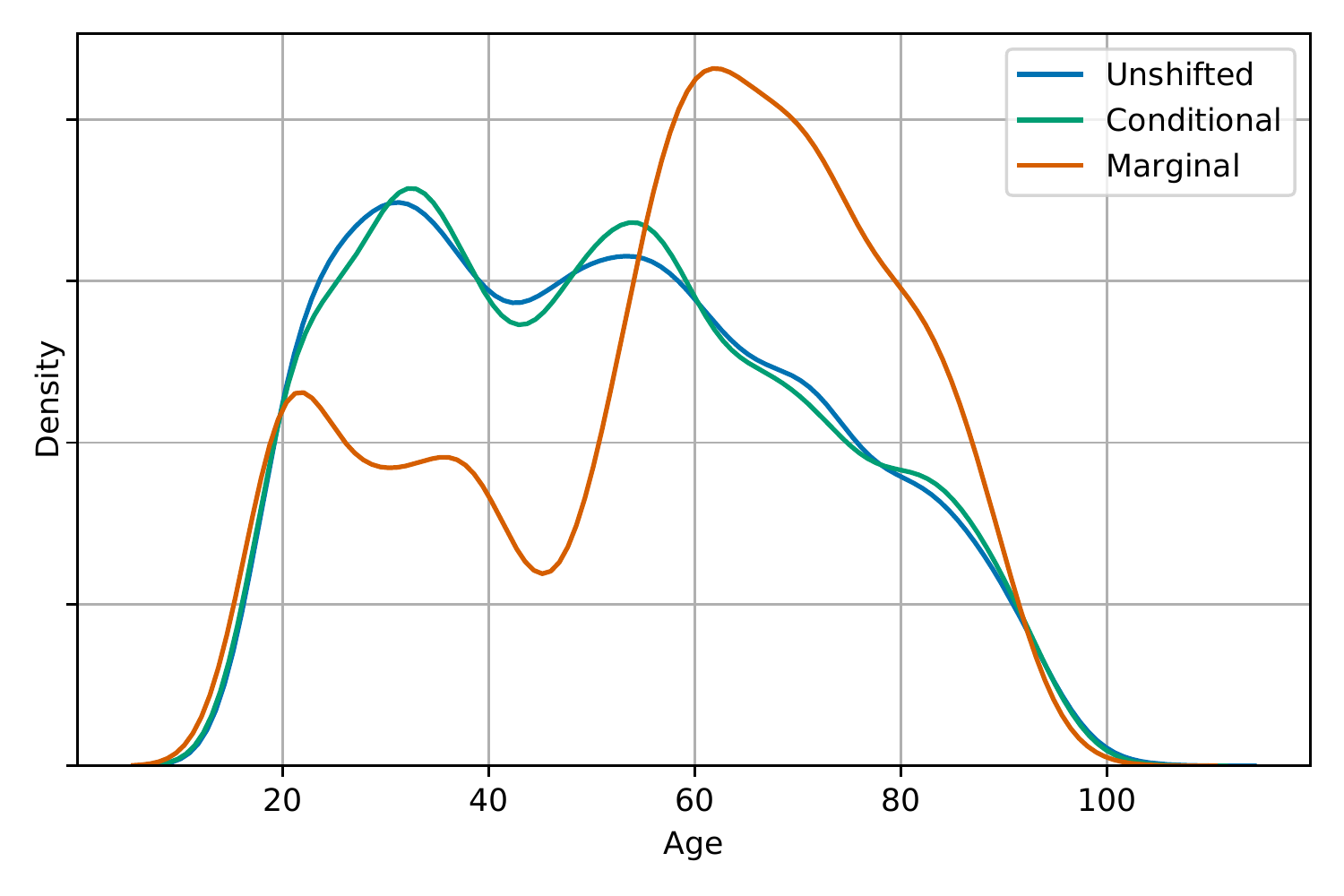}
 \caption{Age distribution of the worst-case subsamples resulting from no shift (Unshifted), shifts in the marginal distribution of features (Marginal), and shifts in the way tests are ordered (Conditional) for subsample size $(1-\alpha)= 0.6$. The age distribution changes substantially under covariate shift.}
  \label{fig:age}
  \vspace{-0.2in}
\end{figure}

One approach to stability analysis is through the lens of \emph{distributionally robust optimization} (DRO) \citep{ben2013robust,duchi2016statistics}. Instead of examining a model's performance only on the (empirical) test distribution associated with a particular validation dataset, DRO defines an \emph{uncertainty set} of possible test distributions and considers the model's performance on the \emph{worst-case} distribution chosen from this set. The distributionally robust perspective allows us to identify the types of environments which lead to poor model performance. To establish the stability of a model under the many possible ways in which environments can change, we need to be able to evaluate worst-case performance under a corresponding variety of shifts. This requires a flexible framework for specifying how the data distribution can change that can reflect targeted changes in environment. Previous shift formulations which consider, e.g., shifts in the entire joint distribution $P(X,Y)$ (e.g., \citet{ben2013robust,duchi2016statistics}) or shifts in the covariate distribution $P(X)$ (e.g., \citet{duchi2019distributionally,chen2016robust}), cannot express finer-grained shifts which isolate changes in the characteristics of a population from decisions made based on those characteristics.

To illustrate, suppose we were interested in analyzing the stability of a diagnosis model's performance to changes in the way clinicians order tests, while keeping the underlying patient population fixed. If the model's predictions depend heavily on the results of a particular test and it is deployed to a hospital where that test is not common, the model could become unsafe to use. This requires specifying a shift in the distribution of test orders $P(\textnormal{test order} \mid \textnormal{age, patient history})$. Previous formulations that do not keep the patient population fixed (e.g., covariate shift) may identify different indicators of poor performance than what we seek to evaluate. To see this (Fig \ref{fig:age}), we evaluated a real diagnostic model's robustness to covariate shift, which resulted in a worst-case age distribution (orange) that differed substantially from the observed age distribution (blue). By using the more flexible framework described in this work, we were able to evaluate a worst-case change in the distribution of test orders that did not affect the age distribution (green).

In this paper, we develop a method for analyzing the stability of models to dataset shift without requiring the collection of new datasets. To evaluate a model under different shifts, users specify a set of \textit{immutable} variables whose distribution should remain fixed, a set of \textit{mutable} variables whose distribution is allowed to change, and a population proportion. The method then identifies the subpopulation of the specified proportion with the worst average loss. This subpopulation is chosen based on the values of the mutable and immutable variables, but is constrained so that the distribution of the immutable variables matches that of the full sample. If the model performs well on this subpopulation, this indicates robustness to the specified type of shift. Conversely, if the method identifies a subpopulation with low performance, this subpopulation may be used to guide further model development or as part of a ``warning label'' for users of the model.

\textbf{Contributions:} First, by defining shifts in conditional distributions, we generalize previous DRO formulations which consider shifts in marginal or joint distributions (Section \ref{sec:prob-form}). This extends our ability to analyze stability to realistic scenarios that were not possible before. Second, we propose the first $\sqrt{N}$-consistent method for estimating the worst-case risk under these distribution shifts (Section \ref{sec:theory}). Finally, on a medical prediction task, we demonstrate that this method can be used to compare the stability of models and identify conditions under which a model would be unsafe to use.

\section{Methods}\label{sec:prob-form}

We are interested in evaluating a prediction algorithm
$\mathcal{M}:\mathcal{X}\rightarrow \mathcal{Y}$ which has been trained to
predict a target variable $Y$ with support set $\mathcal{Y}$ from a set of covariates
$X$ with support set $\mathcal{X}$. 
As a running example, consider evaluating a model
for diagnosing a disease (i.e., $Y$ is a binary
label for the presence of the disease) from a patient's medical history.
As in the case of a third
party reviewer,
we will assume that $\mathcal{M}$ is fixed and that we are evaluating the performance of
$\mathcal{M}$ on a fixed test dataset $\mathcal{D} =
\{(x_i,y_i)\}_{i=1}^N$ drawn i.i.d. from some test distribution $P$. 
Classically, to evaluate $\mathcal{M}$, we would select a loss function $\ell: \mathcal{Y}\times\mathcal{Y} \rightarrow \mathbb{R}$ and estimate the expected loss under the test distribution $P$ as $\mathbb{E}_P[\ell(Y,\mathcal{M}(X))] \approx \frac{1}{N}\sum_i \ell(y_i,\mathcal{M}(x_i))$. However, in addition to the expected loss under $P$, in many practical applications we would like to know how robust the model's expected loss is to changes or differences in $P$, referred to as distribution shift. 
In this section, we describe how to estimate the expected loss of $\mathcal{M}$ under worst-case shifts in the distribution $P$.
We proceed by formally defining distribution shifts, specifying our objective for evaluating model performance under shifts, and, finally, deriving an estimator for this performance.

\subsection{Defining Distribution Shifts}
\label{sec:defining_shifts}
To define general distribution shifts, we will partition the variables into three sets. Let $Z \subset \{X, Y\}$ be a set of \emph{immutable} variables whose marginal distribution should remain fixed, $W \subset \{X, Y\}\setminus Z$ be a set of \emph{mutable} variables whose distribution (given $Z$) we allow to vary, and $V = \{X, Y\} \setminus (W \cup Z)$ be the remaining \emph{dependent} variables. Let $\mathcal{Z}$, $\mathcal{W}$, and $\mathcal{V}$ be their respective support sets. This partition of the variables defines a factorization of $P$ into $P(V \mid W,Z) P(W \mid Z) P(Z)$. Then, we consider how model performance changes when $P(W \mid Z)$ is replaced with a new distribution $Q(W \mid Z)$, while leaving $P(Z)$ and $P(V \mid W, Z)$ unchanged. Notably, this formulation generalizes other commonly studied instances of distribution shift. For example, if we let $Z=\emptyset$ and $W = X$, then this corresponds to a \emph{marginal} \citep{duchi2019distributionally} or \emph{covariate shift} \citep{shimodaira2000improving,sugiyama2007covariate}.

The choice of $Z$ and $W$ determines the type of shift we are interested in. Returning to our diagnosis example, setting $Z=\emptyset$ and $W$ to the set of patient demographic and history variables allows us to evaluate what would happen if the patient population were to change. On the other hand, we may also wish to know how our diagnostic algorithm will perform under changes in treatment policies employed by hospitals when the underlying patient population remains the same. For example, clinicians often order tests (e.g., blood tests) in order to inform diagnoses. By setting $W$ to the binary indicator for a particular test order and $Z$ to include patient information used by doctors to decide on a test order, we can evaluate how performance varies under changes in the way clinicians order this test.

\begin{remark}[Shifts supported by the data]
The types of shifts that we can consider with a single dataset are fundamentally limited by the amount of variation observed within that dataset. While we can change the rates at which events occur in our data, we cannot evaluate performance in situations that we have never seen. For example, we cannot reliably estimate how a model will perform in an intensive care unit (ICU) using only data from an emergency department (ED) as there are events that can occur in an ICU that cannot occur in an ED.
\end{remark}

\begin{remark}[Connection to casual inference]
Under certain conditions, shifts in a conditional distribution $P(W \mid Z)$ have an important interpretation as evaluating stability to causal \emph{policy interventions} or \emph{process changes} \citep{pearl2009causality}. That is, the effects of the shift correspond to how the distribution would change under an intervention that changes the way $W$ is generated. Informally, such an interpretation is possible when $Z$ contains sufficient variables to adjust for potential confounding between $W$ and its causal descendants in $V$. A formal description of this condition in the context of causal models is given in the Appendix.
\end{remark}

\subsection{Quantifying Performance Under Shifts}\label{sec:toy}
Our goal is to evaluate the stability of $\mathcal{M}$ to changes in $P(W\mid Z)$. We will do so by identifying subpopulations (and, thus, choices of $Q(W\mid Z)$) on which the model performs poorly. Specifically, for a user specified sample proportion $(1-\alpha)\in (0,1]$, our goal is to find the subpopulation of proportion $(1-\alpha)$ with the highest expected loss such that: (1) the subpopulation is chosen based \textit{only} on variables in $Z$ and $W$, and (2) the distribution of $Z$ in the subpopulation matches that of the full population $P$. We refer to this subpopulation as the \textbf{worst ($1-\alpha$)-subsample}. This subpopulation is a sample from the shifted distribution that corresponds directly to a shift of the type described in Section \ref{sec:defining_shifts}. 

Formally, let $h:\mathcal{W}\times\mathcal{Z} \to [0,1]$ be the function that selects data points as being in or out of the worst ($1-\alpha$)-subsample based on $Z$ and $W$. The expected loss in this subpopulation is given by $\frac{1}{1-\alpha}\mathbb{E}_P[h(W,Z)\mu_0(W,Z)]$, where $\mu_0(W,Z) \equiv \mathbb{E}_{P}[\ell(Y,\mathcal{M}(X)) \mid W,Z]$ is the conditional expected loss given $W$ and $Z$. We will further constrain $h$ so that $\mathbb{E}_P[h(W,Z)\mid Z] = 1-\alpha$ almost everywhere, which \textit{fixes} the distribution of $Z$ in the subpopulation defined by $h$. Then, we are interested in the expected loss given by the worst-case $h$ or
\begin{align}
    \label{eq:cond_primal}
    R_{\alpha,0} \equiv &\sup_{h:\mathcal{W} \times \mathcal{Z}\rightarrow [0,1]} \,\, \frac{1}{1-\alpha} \mathbb{E}_{P}\left[h(W,Z)\mu_0(W,Z)\right]\\
    \label{eq:cond_eq_constraint}
    &\text{s.t.} \quad \,\,\mathbb{E}_P[h(W,Z)\mid Z] = 1-\alpha\quad a.e.
\end{align}
We refer to $R_{\alpha,0}$ as the \textbf{worst-case risk} and it represents the expected loss on the worst ($1-\alpha$)-subsample.

\begin{figure*}[t]
 \vspace{-0.1in}
    \centering
	\begin{subfigure}{0.29\textwidth}
        \centering
		\includegraphics[width=\textwidth]{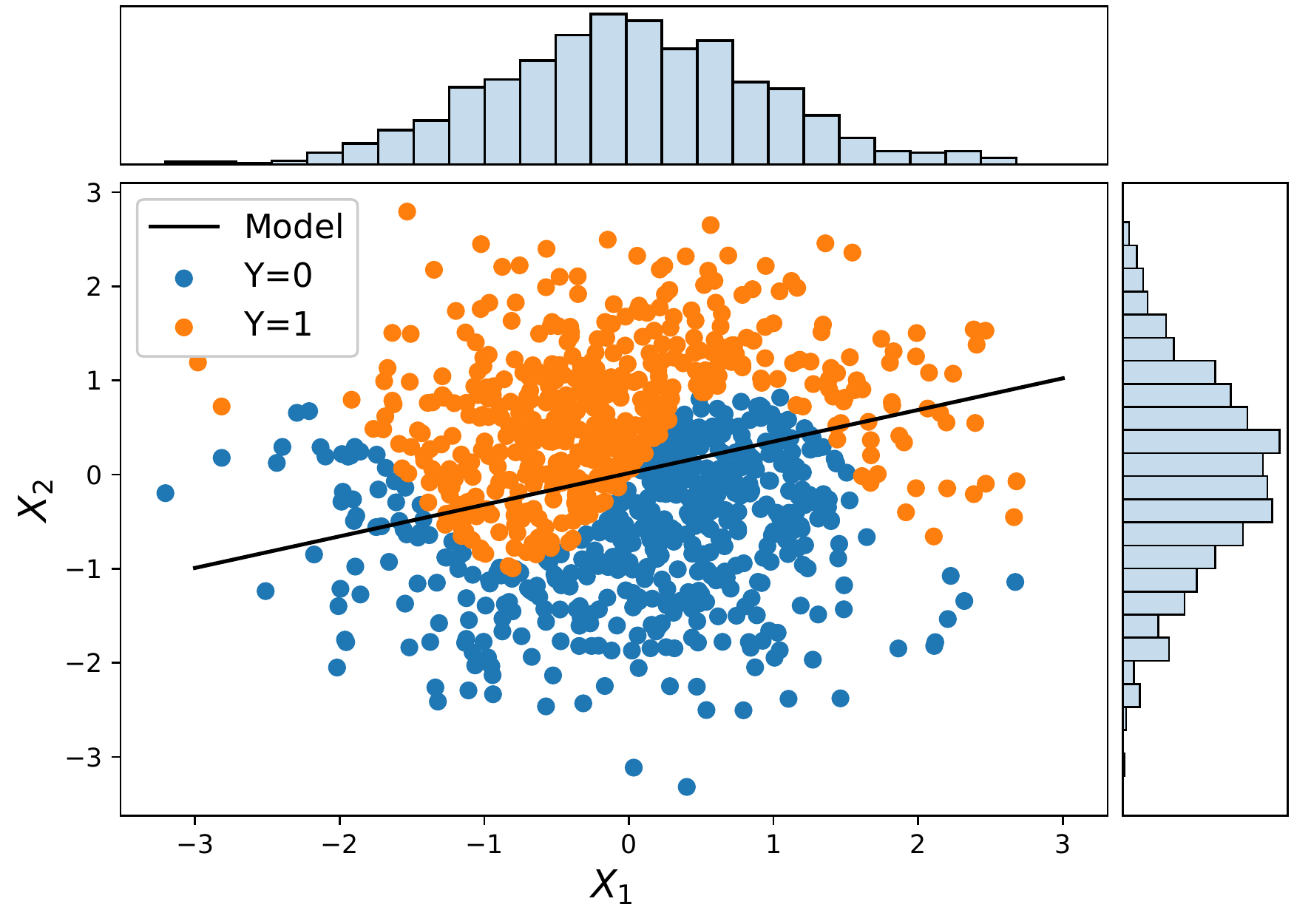}
		\caption{Data distribution}
    \end{subfigure}
    \begin{subfigure}{0.29\textwidth}
        \centering
		\includegraphics[width=\textwidth]{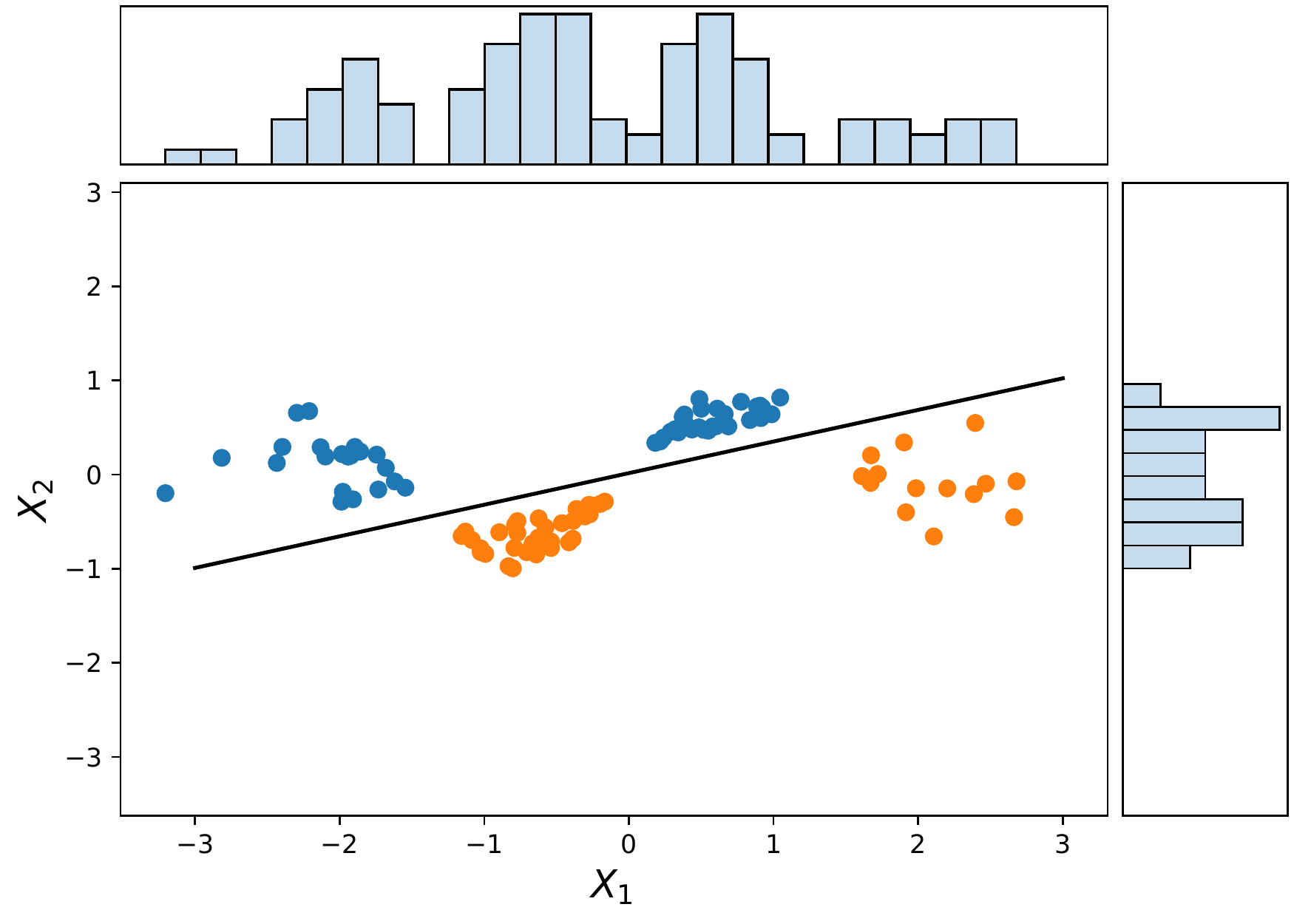}
		\caption{Shift in $P(X_1,X_2)$}
	\end{subfigure}
    \begin{subfigure}{0.29\textwidth}
        \centering
		\includegraphics[width=\textwidth]{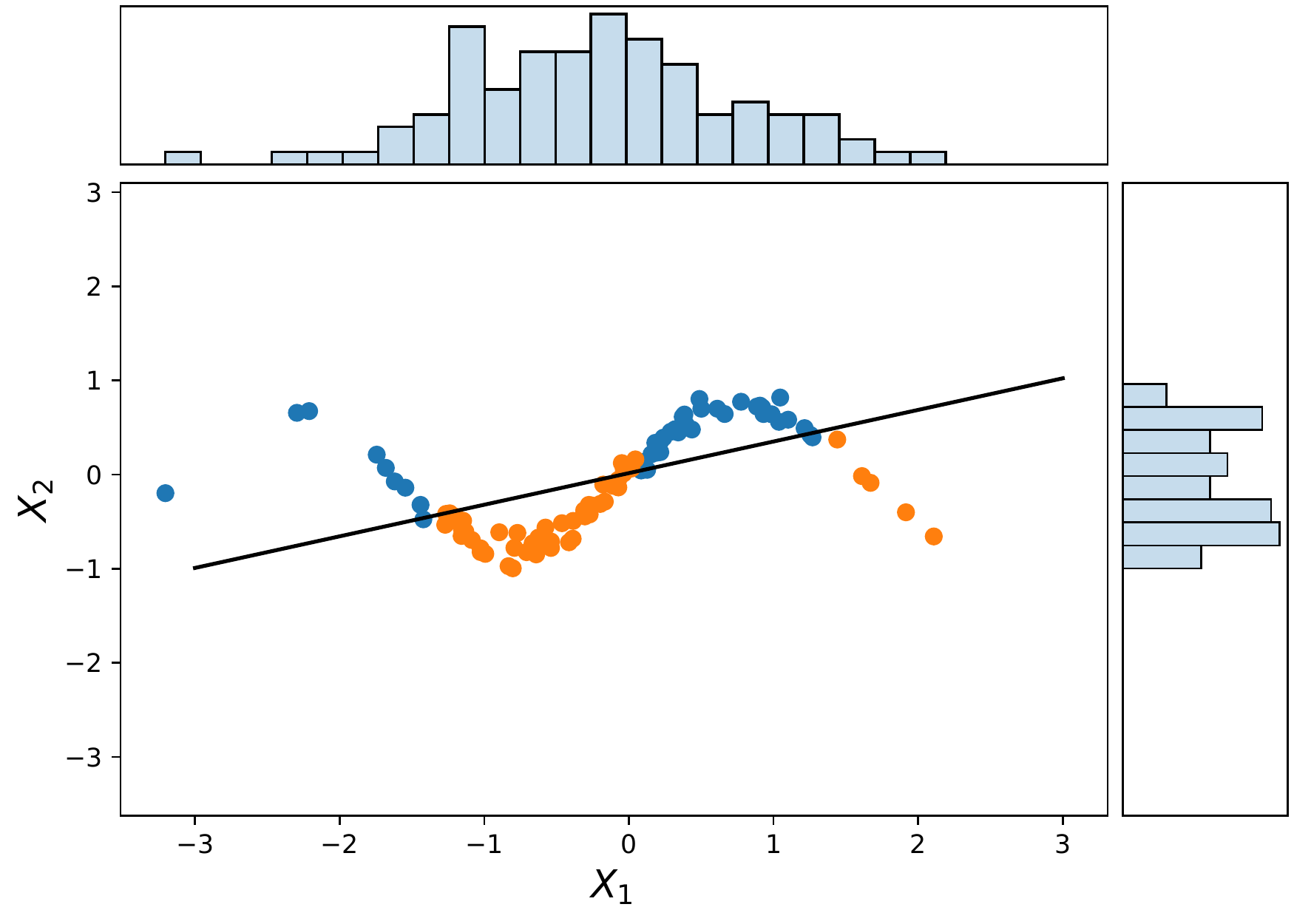}
		\caption{Shift in $P(X_2|X_1)$}
	\end{subfigure}
  \caption{(a) Samples from the data distribution described in Section \ref{sec:toy} along with the decision boundary for a linear classifier fit to this data. (b) The worst ($1-\alpha$)-subsample under a shift in $P(X_1, X_2)$. (c) The worst ($1-\alpha$)-subsample under a shift in $P(X_2|X_1)$. The marginal distributions $P(X_2)$ and $P(X_1)$ are shown on the right and top. The shift in (c) keeps the marginal $P(X_1)$ roughly the same as in (a). Results for $\alpha = 0.9$}
  \label{fig:deterministic_toy}
    \vspace{-0.2in}
\end{figure*}

To illustrate, consider a simple data generating mechanism: $ X_1,X_2 \sim \mathcal{N}(0,1)$, $Y = [X_1 > \sin(2 X_2)]$. Samples from this distribution are shown in Fig \ref{fig:deterministic_toy}a. Suppose that we want to evaluate the stability of a pre-trained linear classifier (also shown in Fig \ref{fig:deterministic_toy}a) under distribution shifts using binary cross-entropy as our loss function. Under a shift in the distribution $P(X_1,X_2)$ (i.e., $W = \{X_1,X_2\}$ and $Z=\emptyset$), Equation \ref{eq:cond_primal} simply selects the $1-\alpha$ fraction of points that yield the highest conditional loss (i.e., incorrectly classified points furthest from the decision boundary) as shown in Fig \ref{fig:deterministic_toy}b. By contrast, under a shift in the distribution $P(X_2 \mid X_1)$ (i.e., $W = \{X_2\}$ and $Z = \{X_1\}$), points producing high conditional loss are chosen subject to the additional constraint of keeping the marginal $P(X_1)$ fixed, as shown in Fig \ref{fig:deterministic_toy}c. This is a direct consequence of the constraint in Equation \ref{eq:cond_eq_constraint} and is the primary difference between the two shifts.

\begin{remark}[Connection to DRO]
In distributionally robust optimization (DRO), a model is trained to minimize the expected loss under the worst-case distribution within an uncertainty set of possible distributions. In \citet{duchi2018learning}, this uncertainty set was defined by specifying a statistical divergence and placing a fixed radius divergence ball around the data distribution. The worst-case risk defined in Equation \ref{eq:cond_primal} can equivalently be thought of as maximizing over distributions in such an uncertainty set defined by the divergence $D_{\infty}(Q \parallel P) = \log \sup_{A\in\mathcal{F}} \frac{Q(A)}{P(A)}$ where $\mathcal{F}$ is the event space for $P$ and $Q$ \citep{van2014renyi}.\footnote{We give a full mapping between the sample proportion and the radius of the divergence ball in the Appendix.} A similar connection was previously described in \citet{duchi2019distributionally}.
\end{remark}

\subsection{Estimating the Worst-Case Risk}\label{sec:estimator}
Having defined the worst-case risk $R_{\alpha,0}$ in Equation \ref{eq:cond_primal}, we now turn to the problem of estimating this risk in finite samples. While previous work on distributionally robust optimization has considered how to \emph{train} models that minimize (upper bounds on) worst-case risk (e.g., \citet{duchi2018learning}), to the best of our knowledge the problem of accurately \emph{estimating} the worst-case risk itself has not received much attention. Thus, in this section we present a consistent estimator for $R_{\alpha,0}$ to address this problem. In Section \ref{sec:theory}, we will show that this estimator has central limit properties, allowing us to easily estimate confidence intervals in settings with high-dimensional or continuous features.

Our estimator relies heavily on the dual formulation for Equation \ref{eq:cond_primal}. Defining $\mu_0$ to be the conditional expected loss as before, and following \citet{duchi2019distributionally} and \citet{duchi2018learning}, the dual is given by
\begin{align}
	\label{eq:dual}
	R_{\alpha,0} = \mathbb{E}_{P}\left[\frac{1}{1-\alpha}(\mu_0(W,Z) - \eta_0(Z))_+ + \eta_0(Z)\right]
\end{align}
where $(x)_+ = \max \{x,0\}$ is the ramp function and the function $\eta_0: \mathcal{Z} \to \mathbb{R}$ is given by
\begin{align}
	\label{eq:eta0}
	\eta_0 = \arginf_{\eta: \mathcal{Z} \to \mathbb{R}} \mathbb{E}_{P}\left[\frac{1}{1-\alpha}(\mu_0(W,Z) - \eta(Z))_+ + \eta(Z)\right]
\end{align}

A full derivation of Equation \ref{eq:dual} can be found in the Appendix. Note that the objective for $\eta_0$ is equivalent to the mean absolute deviation objective used in quantile regression. Indeed, $\eta_0$ is the conditional quantile function of the conditional expected loss $\mu_0$. Thus, estimating $\eta_0$ given $\mu_0$ is exactly quantile regression.\footnote{When $Z = \emptyset$, the difficulty of the estimation problem is significantly reduced because $\eta_0$ is a scalar.} Estimating $R_{\alpha,0}$ requires first estimating the two (potentially infinite dimensional) nuisance parameters $\mu_0$ and $\eta_0$. The dual formulation for $R_{\alpha,0}$ suggests a potential naive estimation procedure: (1) Estimate $\hat{\mu} \approx \mu_0$, (2) estimate $\hat{\eta} \approx \eta_0$ by plugging $\hat{\mu}$ into Equation \ref{eq:eta0}, and (3) estimate $\hat{R}_\alpha \approx R_{\alpha,0}$ by plugging $\hat{\mu}$ and $\hat{\eta}$ into Equation \ref{eq:dual}. When $Z$ and $W$ are low-cardinality discrete variables and $\gamma_0 = (\mu_0,\eta_0)$ is a low-dimensional vector, this simple procedure is $\sqrt{N}$-consistent, as shown in Proposition 1 of  \cite{duchi2019distributionally} for the case of $Z=\emptyset$. In most practical scenarios, $Z$ or $W$ may be continuous or high-dimensional and we will instead want to use flexible machine learning (ML)-based estimators (e.g., random forests or deep neural nets) for $\gamma_0$. However, \citet{chernozhukov2018double} showed that, due to the slow convergence rates and common practice of regularization in flexible ML methods, their use in plugin procedures like this can lead to substantial bias and poor convergence rates.

\begin{algorithm}[t]
    \SetAlgoLined
    \KwIn{Model $\mathcal{M}$, Dataset $\mathcal{D} = \{(w_i,z_i,v_i)\}_{i=1}^n$, and $K$ cross-validation folds $I_k\subset \{1,\dots,n\}$ and $I^c_k = \{1,\dots,n\} \setminus I_k$}
    \For{$k=1,\dots,K$}{
        Estimate $\hat{\mu}_k \approx \mu_0$ using data in $I^c_k$\\
        Estimate $\hat{\eta}_k \approx \eta_0$ according to Eq. \ref{eq:eta0} using $\hat{\mu}_k$ and data in $I^c_k$\\
		\For{$i \in I_k$}{
			Let $\hat{\mu}_{i} = \hat{\mu}_k(w_i,z_i)$\\
			Let $\hat{\eta}_{i} = \hat{\eta}_k(z_i)$\\
			Let $\hat{h}_{i} = [\hat{\mu}_{i} > \hat{\eta}_{i}]$\\
		}
    }
	
    Let $\hat{R}_\alpha = \frac{1}{K} \sum_{k} \frac{1}{|I_k|}\sum_{i \in I_k} \frac{1}{1-\alpha}(\hat{\mu}_{i} - \hat{\eta}_{i})_+ + \hat{\eta}_{i}$\\
    \qquad\qquad\qquad\qquad$+ \frac{1}{1-\alpha}\hat{h}_{i}(\ell(y_i,\mathcal{M}(x_i)) - \hat{\mu}_{i})$\\
    \KwResult{$\hat{R}_\alpha$, $\{\hat{h}_i\}_{i=1}^n$}
    \caption{\textsc{Worst-case Sampler}}
    \label{alg:alg1}
\end{algorithm}

To avoid these issues, following \citet{chernozhukov2018double} and \cite{jeong2020robust}, we propose a so-called ``debiased'' machine learning (DML) estimator for $R_{\alpha,0}$. Among other conditions we will prove in Section \ref{sec:theory}, this estimator maintains $\sqrt{N}$-consistency \textit{without} assuming $\sqrt{N}$ convergence rates for the estimators used for $\mu_0$ and $\eta_0$. The \textsc{Worst-case Sampler}, detailed in Algorithm \ref{alg:alg1},\footnote{Code available at \url{https://github.com/asubbaswamy/stability-analysis}.} splits the data into $K$ folds, estimates $\mu_0$ and $\eta_0$ on each fold using ML, and then combines these estimates in a way that adjusts for the slower convergence rates of the estimators for $\mu_0$ and $\eta_0$. Within the algorithm, $I_k$ represents the $k$'th fold, $\hat{\mu}_i \approx \mu_0(w_i,z_i)$ and $\hat{\eta}_i \approx \eta_0(z_i)$ are the nuisance parameter estimates for instance $i$ and $\hat{h}_{i} = [\hat{\mu}_{i} > \hat{\eta}_{i}]$ implements the indicator function that selects the worst $(1-\alpha)$-subsample. Given these, we estimate $\hat{R}_\alpha \approx R_{\alpha,0}$, the risk on the worst $(1-\alpha)$-subsample, as 
\if1\forarxiv
\begin{align}
\label{eq:r_hat}
    \hat{R}_{\alpha} = &\frac{1}{K} \sum_{k} \frac{1}{|I_k|}\sum_{i \in I_k} \frac{1}{1-\alpha}\left((\hat{\mu}_{i} - \hat{\eta}_{i})_+ + [\hat{\mu}_i \geq \hat{\eta}_i](\ell(y_i,\mathcal{M}(x_i)) - \hat{\mu}_{i})\right) + \hat{\eta}_{i}
\end{align}
\fi
\if0\forarxiv
\begin{equation}
\label{eq:r_hat}
    \begin{split}
    \hat{R}_{\alpha} = &\frac{1}{K} \sum_{k} \frac{1}{|I_k|}\sum_{i \in I_k} \frac{1}{1-\alpha}\left((\hat{\mu}_{i} - \hat{\eta}_{i})_+ \right. \\ 
    &\left. + [\hat{\mu}_i \geq \hat{\eta}_i](\ell(y_i,\mathcal{M}(x_i)) - \hat{\mu}_{i})\right) + \hat{\eta}_{i}.
    \end{split}
\end{equation}
\fi
Next, we establish the correctness of this estimator before applying it to analyze the stability of models on a real clinical diagnosis problem.

\section{Results}\label{sec:results}

In this section, we present three main results: First, we prove that the proposed estimation method has properties that allow it to reliably estimate the worst-case expected loss under distributional shift (Section \ref{sec:theory}). We also discuss confidence interval estimation, sample size considerations, and the method's limitations. Second, in the context of a practical domain, we validate that the worst-case performance estimated by the method provides meaningful information about the performance in an actual new environment (Section \ref{sec:exp2}). Third, we demonstrate that our method can be used to determine settings in which prediction models may be unsafe to use due to poor performance (Section \ref{sec:exp1}). The proposed method fills the gap created by the lack of tools for performing stability analyses to the many types of shifts that we can encounter in practice.

\subsection{Theoretical Results}\label{sec:theory}
In order to reliably use our method to evaluate model robustness and safety, it is important that (i) our estimator converges to the true worst-case loss (consistency) at reasonable rates, and (ii) we are able to account for statistical uncertainty via, for example, confidence interval estimates. We show this by proving that, despite our use of regularized ML methods, our estimator converges to the true worst-case loss at the same asymptotic rate as had we known the \textit{true} nuisance parameter values ($\sqrt{N}$-consistency). Further, we show that the estimator is asymptotically normal, allowing for easy construction of valid confidence intervals. \emph{This represents the first such estimator for worst-case expected loss under distributional shift in non-trivial settings.}

We make the following assumptions, where $\|\cdot\|_{P,q} = (\mathbb{E}_P[|\cdot|^q])^{1/q}$, $\eta(\mu)$ is the true conditional quantile function for $\mu$ (e.g., $\eta(\mu_0) = \eta_0$), and, for parsimony, we drop the dependence of $\hat{\gamma}$, $\gamma_0$, and $\ell$ on $W$, $Z$, and $V$.
\if1\forarxiv
\begin{assumption} \label{ass:conv_rate}
    Let $(\delta_N)_{n=1}^\infty$ and $(\Delta_N)_{n=1}^\infty$ be sequences of positive constants approaching $0$, let $c$, $C$, and $q$, be fixed strictly positive constants such that $q > 2$, and let $K \geq 2$ be a fixed integer. Also, let $\hat{\gamma}_k = (\hat{\mu}_k,\hat{\eta}_k)$ be the estimates of $\gamma_0 = (\mu_0,\eta_0)$ estimated on the $k$'th cross-validation fold. Then, we assume that (a) $\|\ell\|_{P,q} \leq C$ and with probability no less than $1-\Delta_N$:
    \begin{enumerate}
        \item[(b)] $\|\hat{\gamma}_k - \gamma_0\|_{P,q} \leq C$,
        \item[(c)] $\|\hat{\gamma}_k - \gamma_0\|_{P,2} \leq \delta_N$,
        \item[(d)] $\|\hat{\mu}_k - \mu_0\|_{P,\infty} \leq \delta_N N^{-1/3}$,
        \item[(e)] $\|\eta(\hat{\mu}_k) - \hat{\eta}_k\|_{P,\infty} \leq \delta_N N^{-1/3}$, and 
        \item[(f)] for $r \in [0,1]$ and $\mu_{r,k} = \mu_0 + r(\hat{\mu}_k - \mu_0)$, there exists a positive density at $\eta(\mu_{r,k})(Z)$ almost everywhere.
    \end{enumerate}
\end{assumption}
\fi
\if0\forarxiv
\begin{assumption} \label{ass:conv_rate}
    Let $(\delta_N)_{n=1}^\infty$ and $(\Delta_N)_{n=1}^\infty$ be sequences of positive constants approaching $0$, let $c$, $C$, and $q$, be strictly positive constants such that $q > 2$, and let $K \geq 2$ be a fixed integer. Also, let $\hat{\gamma}_k = (\hat{\mu}_k,\hat{\eta}_k)$ be the estimates of $\gamma_0 = (\mu_0,\eta_0)$ on the $k$'th cross-validation fold. Then, we assume that (a) $\|\ell\|_{P,q} \leq C$ and with probability $\geq 1-\Delta_N$: (b) $\|\hat{\gamma}_k - \gamma_0\|_{P,q} \leq C$, (c) $\|\hat{\gamma}_k - \gamma_0\|_{P,2} \leq \delta_N$, (d) $\|\hat{\mu}_k - \mu_0\|_{P,\infty} \leq \delta_N N^{-1/3}$, (e) $\|\eta(\hat{\mu}_k) - \hat{\eta}_k\|_{P,\infty} \leq \delta_N N^{-1/3}$, and (f) for $r \in [0,1]$ and $\mu_{r,k} = \mu_0 + r(\hat{\mu}_k - \mu_0)$, there exists a positive density at $\eta(\mu_{r,k})(Z)$ almost everywhere.
\end{assumption}
\fi
\begin{remark}[Interpreting Assumption 1]
The assumptions of most concern here are (b) and (c), which guarantee consistency of the nuisance parameter estimators, as well as (d) and (e), which bound the convergence rates of said estimators. If (b) or (c) are violated, possibly due to model misspecification, then $\hat{R}_\alpha$ will not be consistent. If (d) or (e) are violated, then $\hat{R}_\alpha$ may remain consistent, but the estimator variance given in Equation \ref{eq:variance}, and, by extension, any confidence intervals, may be incorrect. Importantly, the $N^{-3}$ rates in (d) and (e) are slower than the $\sqrt{N}$ rate we desire. This admits the use of various ML estimators for $\mu_0$ and $\eta_0$, such as ReLU neural networks.\footnote{We refer readers to the Appendix of \citet{jeong2020robust} for a discussion of other estimators.}
Assumption \ref{ass:conv_rate} (a) ensures that there are no substantial regions of the feature space with infinite expected loss. Finally, Assumption \ref{ass:conv_rate} (f) is a standard requirement for estimating quantiles which ensures that the conditional quantiles of $\hat{\mu}_k$ converge to the conditional quantiles of $\mu_0$ \citep{van2000asymptotic,jeong2020robust}.
\end{remark}

We can now guarantee the $\sqrt{N}$-consistency and central limit properties of $\hat{R}_{\alpha}$ using a version of Theorem 3.1 from \citet{chernozhukov2018double}:
\begin{theorem}\label{thm:conv}
    Under Assumption \ref{ass:conv_rate}, let $\{\delta_N\}_N$ be a sequence of positive integers converging to zero such that $\delta_N \geq N^{-1/2}$ for all $N \geq 1$. Then we have that $\hat{R}_\alpha$ concentrates in a $1/\sqrt{N}$ neighborhood of $R_{\alpha,0}$ and is approximately linear and centered Gaussian:
    \if1\forarxiv
    \begin{align}
        \sqrt{N}\sigma^{-1}(\hat{R}_\alpha - R_{\alpha,0}) = \frac{1}{\sqrt{N}} \sigma^{-1} \sum_i \psi(O_i;R_{\alpha,0},\gamma_0) + \mathcal{O}_P(\delta_N) \rightsquigarrow \mathcal{N}(0,1)
    \end{align}
    \fi
     \if0\forarxiv
    \begin{equation}
    \begin{split}
        \sqrt{N}\sigma^{-1}(\hat{R}_\alpha - R_{\alpha,0}) &= \frac{1}{\sqrt{N}} \sigma^{-1} \sum_i \psi(O_i;R_{\alpha,0},\gamma_0)\\
        &+ \mathcal{O}_P(\delta_N) \rightsquigarrow \mathcal{N}(0,1)
    \end{split}
    \end{equation}
    \fi
    where $O_i = (W_i,Z_i,V_i)$, $\gamma_0 = (\mu_0,\eta_0)$,
    \begin{align*}
        \psi(\cdot ; R_{\alpha,0},\gamma_0) = &\frac{1}{1-\alpha}\left((\mu_0 - \eta_0)_+ + [\mu_0 \geq \eta_0](\ell - \mu_0)\right)\\ &+ \eta_0 - R_{\alpha,0},
    \end{align*}
    and $\sigma^2 = \mathbb{E}_P[\psi^2(O;R_{\alpha,0},\gamma_0)]$.
\end{theorem}
An important consequence of this result is that we can further estimate the variance of $\sqrt{N}(\hat{R}_\alpha - R_{\alpha})$ as
\begin{align}\label{eq:variance}
    \hat{\sigma}^2 = \frac{1}{K}\sum_k \frac{1}{|I_k|}\sum_{i\in I_k} \psi^2(O_i;\hat{R}_\alpha,\hat{\gamma}_k),
\end{align}
and this estimate can be used to construct valid $100(1-a)\%$ confidence intervals as $(\hat{R}_\alpha \pm \Phi^{-1}(1-a/2)\sqrt{\hat{\sigma}^2/N})$ \citep{chernozhukov2018double}. 

\begin{remark}[Bias and variance for $\alpha \to 1$]
Note that the variance estimate $\hat{\sigma}^2$ scales with $1/(1-\alpha)^2$, highlighting that a large dataset may be needed to estimate $R_{\alpha,0}$ for $\alpha$ close to $1$. Additionally, if regularization is used when estimating $\mu_0$ and $\eta_0$, this may smooth over low-probability regions with high loss, resulting in potential underestimation of the worst-case risk for $\alpha$ close to $1$ unless a similarly large dataset is used.
\end{remark}

\begin{remark}[Discrete $W$]
When $W$ contains only discrete random variables, Assumption \ref{ass:conv_rate} (f) is not satisfied. In such cases, inspired by \cite{machado2005quantiles}, we augment the \textsc{Worst-case Sampler} (Alg \ref{alg:alg1}) by adding small amounts of independent uniform noise to each $\hat{\mu}_k$ before fitting $\hat{\eta}_k$ and using a slightly augmented version of Equation \ref{eq:r_hat}. This augmented procedure satisfies Assumption \ref{ass:conv_rate} (f) at the cost of an arbitrarily small, user-controlled amount of bias. Full details and theoretical results are described in the Appendix.
\end{remark}

With a reliable estimator in hand, we now demonstrate its utility on a real clinical prediction problem.

\subsection{Experimental Results} \label{sec:experiments}
In the context of a practical domain, we now demonstrate that: (a) the proposed method correctly estimates the performance of a model under adversarial distribution shifts and (b) the proposed method can be used to compare the stability of multiple models and determine shifted settings in which a model may be unsafe to use. Suppose we are third-party reviewers, such as the FDA, who wish to evaluate the safety of a machine learning-based clinical diagnostic model to changes in clinician lab test ordering patterns. Our goal is to determine test ordering patterns under which the model is likely to have poor performance. This may serve to guide further model refinement or provide a basis for determining ``warning labels'' about the indications for model use.

\subsubsection{Dataset and Models}
We demonstrate our approach on machine learning models for diagnosing sepsis, a life-threatening response to infection. We follow the setup of \citet{giannini2019machine} who developed a clinically validated sepsis diagnosis algorithm. Our dataset contains electronic health record data collected over four years at Hospital A in our institution's health network. The dataset consists of 278,947 emergency department patient encounters. The prevalence of the target disease, sepsis, is $2.1\%$. 17 features pertaining to vital signs, lab tests, and demographics were extracted. A full characterization of the dataset and features can be found in the Appendix. We evaluate the robustness of the models to changes in test ordering patterns using a held-out sample of 10,000 patients which serves as the evaluation dataset.

We consider two sepsis prediction models: The \textbf{classical} model was trained using classical supervised learning methods, while the \textbf{robust} model was trained using the ``surgery estimator'' \citep{subbaswamy2019preventing}. While both models are random forest classifiers and were trained using the same data, the robust model was trained with the goal of being stable to shifts in test ordering patterns. Because our focus is on evaluating models rather than training them, we refer readers to the Appendix for details about the training procedures.

\textbf{Setup:} To analyze stability to changes in test ordering patterns, we estimate how the classification accuracy\footnote{Accuracy is 1 minus the expected 0-1 loss.} of the classical model changes as we vary $1-\alpha$, the sample proportion. When $(1-\alpha) = 1$, the worst-case subsample is the original dataset and corresponds to no distribution shift. As $1-\alpha$ approaches 0, the worst-case subsample becomes smaller and the worst-case shifted distribution can become increasingly different from the original data distribution. The shift in test ordering patterns can be represented as shifts in the conditional distribution $P(\text{test order} \mid \text{demographics, disease status})$. By specifying demographics and disease status as immutable variables, we fix their distributions such that $P(\text{demographics, disease status})$ in the worst $(1-\alpha)$-subsample is the same as in the original dataset.

\subsubsection{Validation Experiment}\label{sec:exp2}
\vspace{-0.1in}
\begin{figure}[!t]
 \centering
	\includegraphics[scale=0.4]{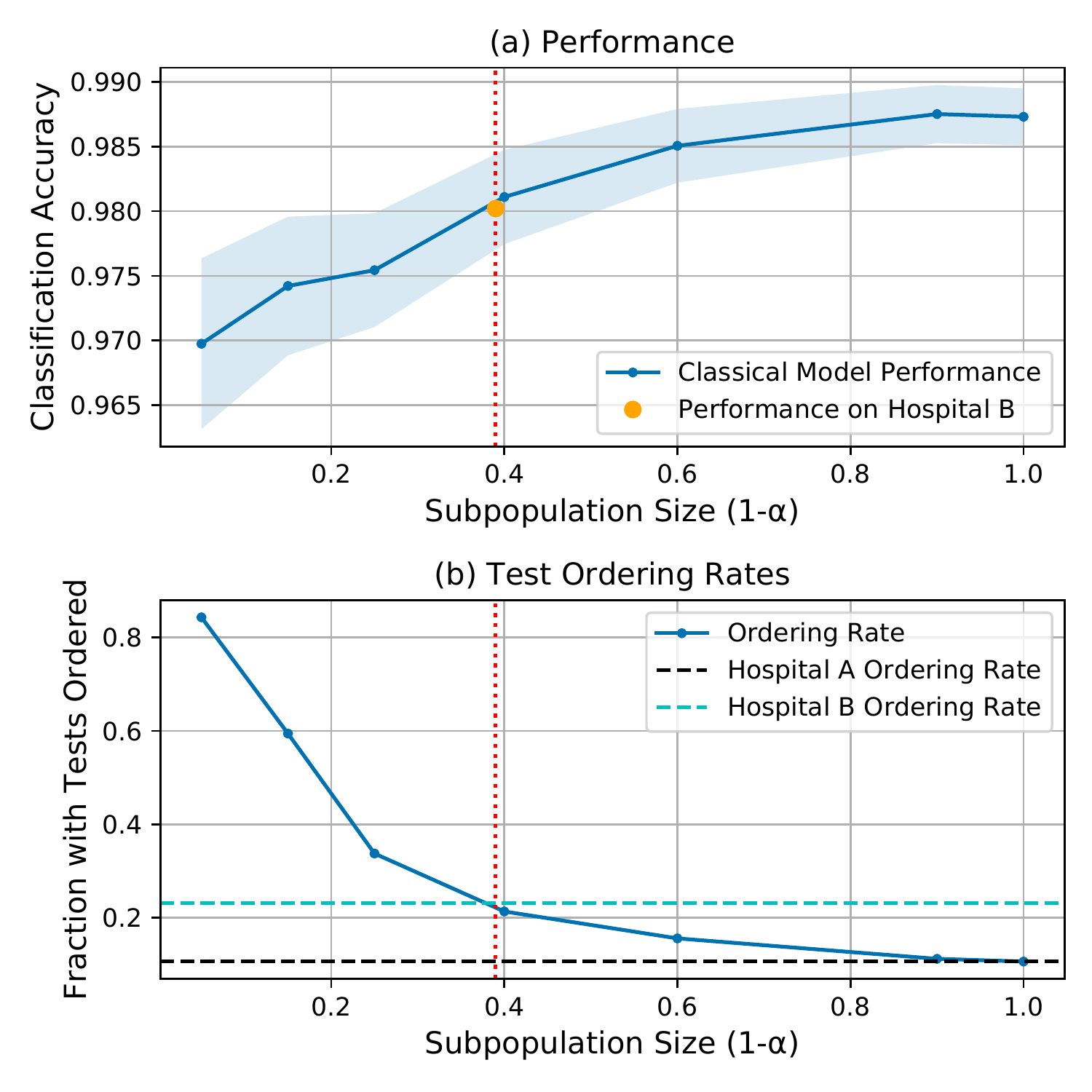}
  \caption{(a) Estimated accuracy of the classical model on its worst $(1-\alpha)$-subsamples (blue curve). The orange dot is the actual accuracy of the classical model when applied to Hospital B. The shaded blue region denotes a $95\%$ confidence interval. (b) The test ordering rate in the worst $(1-\alpha)$-subsamples. The Hospital A test ordering rate is $11\%$ while at Hospital B it is $23\%$.}
  \label{fig:validation}
      \vspace{-0.2in}
\end{figure}

We begin by validating that the method can correctly estimate the performance of a model under worst-case distribution shift, and that this provides meaningful information about the performance in new, shifted environments. To do so, we compare the estimated worst-case performance under a shift in lab test ordering patterns to the observed performance in a new environment exhibiting such a shift. For this purpose, we used an additional dataset containing the same variables collected from a different hospital (Hospital B, also in our institution's health network).\footnote{Description in the supplemental material.} Hospital B has a similar patient population (i.e., demographic makeup and disease prevalence) to Hospital A. On the other hand, the two hospitals differ substantially in the rate of test orders ($11\%$ ordering rate at Hospital A vs $23\%$ ordering rate at Hospital B). Thus, this is exactly the shift in test ordering patterns we wish to study.

Fig \ref{fig:validation}a plots the estimated worst-case performance of the classical model across sample proportions ($1-\alpha$). As expected, the model's performance worsens as the sample proportion decreases. To characterize the worst $(1-\alpha)$-subsamples, in Fig \ref{fig:validation}b we plot the fraction of the subsample that received a test (the test ordering rate) against the sample proportion. Since the test ordering rate doubles to $23\%$ from Hospital A to Hospital B, we expect the accuracy of the model at Hospital B to be greater than or equal to the worst-case accuracy for $(1-\alpha) = 0.39$ (the sample proportion producing an ordering rate of $23\%$). Indeed, we find that the accuracy at Hospital B (orange dot) is worse than the accuracy at Hospital A. Further, the accuracy at Hospital B is roughly equal to the estimated worst-case accuracy for this subpopulation size (and well within the blue shaded 95\% confidence interval). The intervals also demonstrate the relationship between the sample proportion and the estimator variance. If the intervals are too large for large $\alpha$, more data may need to be collected in order to trust the estimates. These results show that the risk curves accurately inform us about how performance changes under worst-case shifts.

\subsubsection{Use Cases for Stability Analysis}\label{sec:exp1}
\vspace{-0.1in}
\begin{figure}[t]
 \centering
	\includegraphics[scale=0.45]{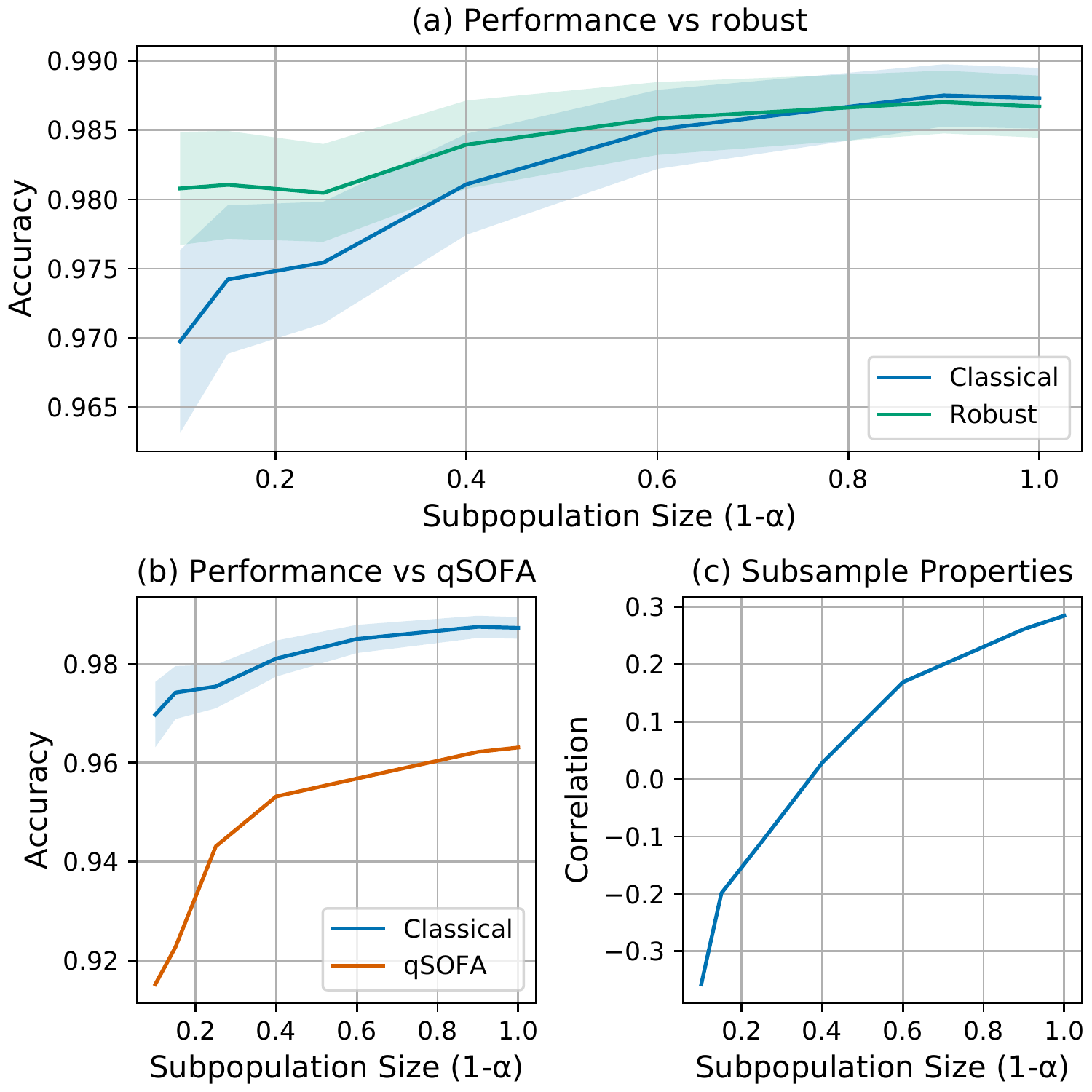}
  \caption{(a) Accuracy of the classical and robust models on their respective worst $(1-\alpha)$-subsamples. Shaded regions denote $95\%$ confidence intervals. (c) Accuracy of the classical model on its worst $(1-\alpha)$-subsamples and accuracy of qSOFA on these same subsamples. (b) Correlation between sepsis and lab test order in the classical worst-case subsamples.}
  \label{fig:conditional}
    \vspace{-0.2in}
\end{figure}

We now consider three use cases of the proposed method which demonstrate how to analyze model stability without requiring additional data gathering.

\textbf{Comparing the stability of models:} Model developers often need to compare the performance of models. We compare the stability of two models with respect to potential changes in test ordering patterns. Fig \ref{fig:conditional}a shows the estimated worst-case risk for the classical and robust models at various sample proportions ($1-\alpha$). The plot confirms that the robust model is more stable under shifts in test ordering patterns. As the subpopulation size decreases, the robust model increasingly outperforms the classical model on their respective worst-case subpopulations (doing so significantly for $(1-\alpha) \leq 0.1$). Model differences can be characterized by their worst performing test ordering patterns.

\textbf{Comparing to baseline standards:} Another important use case is to understand in what settings or on what subpopulations a model's performance becomes worse than existing baseline standards of care. For example, FDA reviewers might want to verify that a sepsis diagnosis model performs at least as well as qSOFA, a simple, established early warning score that is commonly used to detect patient deterioration \citep{singer2016third}. To make this assessment we first estimate the worst-case risk and worst ($1-\alpha$)-subsamples for the classical model and then estimate the accuracy of qSOFA on these same subsamples (shown in Fig \ref{fig:conditional}b). Across all $\alpha$ values the classical model significantly outperforms qSOFA, showing it is safe relative to standard of care under a variety of test ordering patterns.

\textbf{Determining unsafe use conditions:} Finally, model reviewers and developers may wish to understand the settings (i.e., test ordering patterns) associated with poor performance so that they can further improve the model or add warnings about its use. In Fig \ref{fig:validation}b we saw that increased test ordering rates were associated with worse performance. To investigate further, for each worst ($1-\alpha$)-subsample, we plot the correlation between a positive sepsis diagnosis and a test order within this subsample (Fig \ref{fig:conditional}c). There is a clear trend: smaller worst-case subsamples have decreasing correlation between disease status and test orders. From this we can conclude that the classical model will experience performance deterioration when applied to hospitals that widely (as opposed to selectively) test patients. However, relative to the original dataset, even a doubling of the test ordering rate results in only small performance deterioration. Thus, we may conclude that the model is safe to use under a wide range of test ordering patterns, though developers might still seek to improve the model in settings with small or negative test order correlations.

\section{Related Work}\label{sec:related}
We now overview various threads of work on the problem of dataset shift, in which the deployment environment differs from the training environment.

\paragraph{Adapting models to new environments:}
 One of the most common dataset shift paradigms assumes that the deployment environment is known and that we have limited access to data from the deployment environment \citep{quionero2009dataset}. Many works consider  the problem of learning a model using labeled data from the training environment and unlabeled data from the deployment environment, using the unlabeled data to adjust for shifts in $P(X)$ through reweighting (e.g., \citet{shimodaira2000improving,huang2007correcting}) or extracting invariant feature representations (e.g., \citet{gong2016domain, ganin2016domain}). \citet{rai2010domain} assume that we have limited capacity to query the deployment environment and use active learning techniques to adapt a model from the training environment to the deployment environment. While these types of adaptations should absolutely be conducted when possible, our goal in this work is to evaluate how a model will perform in potential future environments from which we do not currently have samples. 

\paragraph{Learning robust models:} 
Another large body of research attempts to proactively improve robustness to dataset shift by \textit{learning} models that are robust to changes or uncertainty in the data distribution. This work falls broadly under the umbrella of distributionally robust optimization (DRO) which, in turn, comes from a large body of work on formulating optimization problems affected by uncertainty (e.g., \citet{ben2013robust,duchi2016statistics, bertsimas2018data}). As in our work, DRO assumes that the true population distribution is in an uncertainty ``ball'' around the empirical data distribution and optimizes with respect to the worst-case such distribution. 
In some cases, the uncertainty set is designed to reflect sampling variability and thus the radius of the ball linearly decreases with the number of samples, but no distributional shift is assumed \citep{namkoong2016stochastic,namkoong2017variance,lei2020convergence}. 
Work on DRO explores a variety ways to define the uncertainty ball/set of distributions. Some have explored balls defined by the so-called f-divergences, which include as special cases KL divergence, $\chi^2$ divergence, and CVaR (used in this work) \citep{lam2016robust,namkoong2016stochastic,namkoong2017variance,duchi2018learning}. 
Still others consider uncertainty sets defined by Reproducing Kernel Hilbert Spaces (RKHS) via Maximum Mean Discrepancy (MMD) \citep{gretton2012kernel,staib2019distributionally,zhu2020kernel} and Wasserstein distances \citep{fournier2015rate,abadeh2015distributionally,sinha2017certifying,esfahani2018data,lei2020convergence}. Unlike approaches using f-divergences, these approaches can allow for distributions with differing support, but they are computationally challenging, often requiring restrictions on the loss or transportation cost functions. Future work may consider extensions of the proposed evaluation framework to MMD and Wasserstein-based uncertainty sets.

A related line of work defines uncertainty sets of environments using causal interventions on the data generating process which allow for arbitrary strength shifts and do not have to be centered around the training distribution \citep{meinshausen2018causality,buhlmann2020invariance}. These methods aim to learn a model with stable or robust performance across the uncertainty set of environments. Using datasets collected from multiple environments, various methods seek invariant feature subsets \citep{rojas2018invariant} or representations \citep{arjovsky2019invariant}. Alternatively, assuming knowledge of the causal graph of the data generating process, \citet{subbaswamy2019preventing} identify mechanisms that can shift and find a stable \emph{interventional} distribution to use for prediction. Assuming linear mechanisms, \citet{subbaswamy2018counterfactual} find a stable feature set that can include counterfactuals. Recent end-to-end approaches relax the need for the graph to be known beforehand by instead learning it from data \citep{zhang2020domain,subbaswamy2020spec}.

\paragraph{Evaluating robustness:} Relatively few works have focused on \textit{evaluating} the robustness of a model to distributional shift. \citet{santurkar2020breeds} proposed an algorithm for generating evaluation benchmarks for subpopulation shifts (new subpopulations that were unseen in the training data) by combining datasets with hierarchical class labels. \citet{oakden2020hidden} considered broad, sometimes manual, strategies to evaluate the performance of a model in subpopulations. Both of these start by constructing subpopulations that have semantic meaning and evaluate the model on each of the subpopulations. Thus, discovering a subpopulation with poor performance is either serendipitous or is guided by domain knowledge. In this work, we take a data-driven approach, starting with a worst-case subpopulation and then exploring its properties.

\paragraph{Estimating optimal treatment subpopulations:} 
Finally, our work is methodologically similar to that of \cite{jeong2020robust} and \cite{vanderweele2019selecting}, who sought to estimate the causal effect of a treatment in the worst- and best-case subpopulations, respectively. Whereas, in our work, we seek to find a subpopulation with the highest expected conditional loss, they seek to find a subpopulation with the highest (lowest) conditional average treatment effect, which can be formulated as the best- (worst-) case average treatment effect under a \textit{marginal} shift. A potential extension of this work is to consider optimal treatment subgroups defined by other types of shifts.

    

\section{Conclusion}

As machine learning systems are adopted in high impact industries such as healthcare, transportation, and finance, a growing challenge is to proactively evaluate the safety of these systems to avoid the high costs of failure. To this end, we proposed a framework and estimation method for proactively analyzing the stability of trained machine learning models to shifts in population or setting without requiring the collection of new datasets (an often costly and time-consuming effort). When sufficient variation is present in the available evaluation data, the proposed framework allows us to understand the changes in population or setting under which the model would be unsafe to use. As demonstrated in the experiments, this now enables us to analyze stability to realistic shifts which correspond to policies that can vary across sites, datasets, or over time. Further, this evaluation can be used to guide model refinement and additional data collection. We envision that procedures like the proposed method should become standard practice for analyzing the stability of models in new settings.

\subsubsection*{Acknowledgements}
The authors thank Berkman Sahiner, Vinay Pai, and David Nahmias for helpful questions that improved the presentation, and Alex Gain for early formative discussions. This publication was supported by the Food and Drug Administration (FDA) of the U.S. Department of Health and Human Services (HHS) as part of a financial assistance award U01FD005942 totaling $\$97,144$ with 100 percent funded by FDA/HHS. The contents are those of the author(s) and do not necessarily represent the official views of, nor an endorsement, by FDA/HHS, or the U.S. Government.


\bibliography{references}

\newpage
\appendix
\onecolumn
\section{Causally Interpreting Distribution Shifts}\label{app:causal}

Under certain conditions, shifts in a conditional distribution $P(W \mid Z)$ have an important interpretation as causal \emph{policy interventions} or \emph{process changes} \citep{pearl2009causality}. That is, the effects of the shift corresponds to how the distribution would change under an intervention that changes the way $W$ is generated. Formally, we have the following:
\begin{proposition}
Suppose the data $(X,Y)$ were generated by a structural causal model (SCM) with no unobserved confounders, respecting a causal directed acyclic graph (DAG) $\mathcal{G}$. Then, for a single variable $W$ and set $Z=nd_{\mathcal{G}}(W)$ (non-descendants of W in $\mathcal{G}$), a policy shift in $P(W \mid Z)$ can be expressed as a policy intervention on the mechanism generating $W$ which changes $P(W \mid pa_{\mathcal{G}}(W))$.
\end{proposition}

\begin{proof}
Within the SCM, we have that $W$ is generated by the structural equation $W = f_w(pa(W), \varepsilon_w)$ (where $\varepsilon$ is a $W$-specific exogenous noise random variable). A policy intervention on $W$ replaces this structural equation with a new function $g_w(pa(W), \varepsilon_w)$, which has the effect of changing $P(W \mid pa(W))$ to some new distribution $Q(W \mid pa(W)$. By the local Markov property, we have that $Q(W\mid pa(W)) = Q(W\mid nd(W))$. Thus, a shift from $P(W\mid nd(W))$ to $Q(W\mid nd(W))$ can be expressed as a policy intervention from $f_w$ to $g_w$.
\end{proof}

This result means that in order to causally interpret distribution shifts, we need to adjust for (i.e., put into $Z$) variables that are relevant to the mechanism that generates $W$. Fortunately, we can place additional variables into $Z$ so long as they precede $W$ in a causal or topological order. 

This result can be extended to the case in which the SCM contains unobserved variables. This is of practical importance because often we do not have all relevant variables recorded in the dataset (i.e., there may be unobserved confounders). In these cases, rather than a DAG, the SCM takes the graphical form of a causal acyclic directed mixed graph (ADMG) \citep{richardson2017nested}. ADMGs have directed edges $(\rightarrow)$ which represent direct causal influence (just like in DAGs), but also have bidirected edges $(\leftrightarrow)$ which represent the existence of an unobserved confounder between the two endpoint variables.

We require one technical definition: we will define the Markov blanket of $W$ in an ADMG to be $mb(W) = pa(dis(W)) \cup (dis(W) \setminus \{W\})$, where $dis$ refers to the \emph{district} of a variable and is the set of variables reachable through entirely bidirected paths.

\begin{proposition}
Suppose the data $(X,Y)$ were generated by a structural causal model (SCM), respecting a causal ADMG $\mathcal{G}$. Then, for a single variable $W$ and set $Z=mb_{\mathcal{G}}(W)$, a policy shift in $P(W \mid Z)$ can be expressed as a policy intervention on the mechanism generating $W$ which changes $P(W \mid mb_{\mathcal{G}}(W))$.
\end{proposition}

\begin{proof}
The proof follows just as before, noting that the local Markov property in ADMGs states that a variable $V$ is independent of all variables preceding $V$ in a topological order conditional on $mb(V)$ \citep[Section 2.8.2]{richardson2017nested}.
\end{proof}

\section{Connection to distributionally robust optimization}
\label{sec:dro}

In distributionally robust optimization (DRO), a model is trained to minimize loss on the worst-case distribution from within an uncertainty set of distributions. Here we show how Equation 1 of the main paper can equivalently be thought of a the expected loss under the worst-case distribution from within just such an uncertainty set. Formally, we define an uncertainty set or ``ball'' $\mathcal{P}_{\rho,Z,\mathcal{W}}$ of possible shifted distributions using a statistical divergence $D(\cdot\parallel\cdot)$ and radius $\rho$:  
\begin{align}
    \mathcal{P}_{\rho,Z,\mathcal{W}} = \{Q:D( Q(W) \parallel P(W \mid Z)) \leq \rho\}.
\end{align}

Note that this uncertainty set depends explicitly on the value of $Z$. We are interested in the expected loss of the model when $P(W \mid Z)$ is replaced by $Q(W \mid Z) = Q_Z \in \mathcal{P}_{\rho,Z,\mathcal{W}}$ that maximizes expected loss, written
\begin{align}
    \label{eq:r_rho}
    R_\rho(\mathcal{M};P) = \mathbb{E}_{P}\left[\sup_{Q_Z\in\mathcal{P}_{\rho,Z,\mathcal{W}}} \mathbb{E}_{Q_Z}
    [\mu_0(W,Z) \mid Z]\right],
\end{align}
where, as in the main paper, $\mu_0(W,Z) = \mathbb{E}_{P}[\ell(Y,\mathcal{M}(X)) \mid W,Z]$ is the conditional expected loss given $W$ and $Z$. By construction of $\mathcal{P}_{\rho,Z,\mathcal{W}}$, $Q$ will never place positive weight on regions where $P$ does not. Calculating $R_\rho$ requires calculating the expected loss under various distributions $Q$. As in previous work on DRO and domain adaptation, we rely on \emph{sample reweighting}, which allows us to estimate expectations under $Q$ using samples from $P$. This is done by reweighting samples from $P$ by the likelihood ratio $q/p$. Specifically, the expected loss under $Q$ can be rewritten as
\begin{align}
	\mathbb{E}_{Q}[\mu_0(W,Z) \mid Z] = \mathbb{E}_P\left[\frac{q(W \mid Z)}{p(W \mid Z)}\mu_0(W,Z) ~\Big\vert~ Z\right].
\end{align}
If $Q$ is quite different from $P$, the variance of importance sampling can be high. This variance is naturally governed by $\rho$, which controls how different $Q$ can be from $P$. In order to consider environments that look \textit{very} different from $P$, a large test dataset may be needed. 

To see how this formulation connects to Equation 1 of the main paper, we will make use of the following lemma which follows directly from Theorem 6 in \citet{van2014renyi}

\begin{lemma}
    For probability measures $P$ and $Q$ defined with respect to the same base measure $\mu$ and with corresponding density functions $p$ and $q$, $\sup_{A \in \mathcal{F}} \frac{Q(A)}{P(A)} \leq c$ if and only if $\frac{q}{p} \leq c$ almost everywhere.
\end{lemma}

Using this lemma, we can rewrite Equation 1 as 

\begin{align}
    &\sup_{Q} \,\,  \mathbb{E}_P
    \left[\mathbb{E}_Q[\mu_0(X)\mid Z]\right]\\\nonumber
    &\text{s.t.} \quad \,\,\frac{q(W \mid Z)}{p(W \mid Z)} \leq \exp(\rho)\quad a.e.
\end{align}

This can, in turn, be rewritten as
\begin{align}
\label{eq:lr_primal}
    &\sup_{Q} \,\,  \mathbb{E}_P
    \left[\frac{q(W \mid Z)}{p(W \mid Z)}\mu_0(X)\right]\\\nonumber
    &\text{s.t.} \quad \,\,\frac{q(W \mid Z)}{p(W \mid Z)} \leq \exp(\rho)\quad a.e.
\end{align}
Define $\exp(\rho) h(w,z) = \frac{q(w \mid z)}{p(w \mid z)}$ and $\exp(\rho) = \frac{1}{1-\alpha}$. Then, the constraint in Equation \ref{eq:lr_primal}, combined with the fact that $p$ and $q$ are both densities and thus are bounded below by zero, translates to $h:\mathcal{X} \to [0,1]$. Further, the constraint that $q$ must integrate to one (or, equivalently, $\mathbb{E}_P[q(W \mid Z)/p(W \mid Z) \mid Z] = 1$ almost everywhere) translates to the constraint $\mathbb{E}_P[h(W,Z) \mid Z] = 1 - \alpha$ almost everywhere. Finally, we can rewrite this optimization problem as
\begin{align}
    \sup_{h:\mathcal{W}\times\mathcal{Z} \to [0,1]} \,\,  &\frac{1}{1-\alpha}\mathbb{E}_P
    \left[h(W,Z)\mu_0(W,Z)\right]\\\nonumber
    \text{s.t.} \quad \,\,&\mathbb{E}_P[h(W,Z)\mid Z] = 1 - \alpha \quad a.e.
\end{align}

\section{Derivation of Equation 3 of the main paper}
\label{app:eq_3_derivation}

To derive Equation 3 of the main paper, we will take the Lagrange dual of Equation 1. Recall that Equation 1 is given by

\begin{align}
    R_{\alpha,0} = &\sup_{h:\mathcal{W} \times \mathcal{Z}\rightarrow [0,1]} \,\, \frac{1}{1-\alpha} \mathbb{E}_P\left[h(W,Z)\mu_0(W,Z)\right]\\
    &\text{s.t.} \quad \,\,\mathbb{E}_P[h(W,Z)\mid Z] = 1-\alpha\quad a.e.
\end{align}

Then the Lagrangian is given by

\begin{align}
    \mathcal{L}(h,\nu) = \frac{1}{1-\alpha} \mathbb{E}_P\left[h(W,Z)\mu_0(W,Z)\right] + \int \nu(z) (1-\alpha - \mathbb{E}_P[h(W,Z)\mid Z=z]) ~dz,
\end{align}

where $\nu:\mathcal{Z} \to \mathbb{R}$ is the function of Lagrange multipliers. By recalling that $p$ is the density function associated with $P$ and by defining $\eta(z) = \nu(z)\frac{1-\alpha}{p(z)}$ we get

\begin{align}
    \mathcal{L}(h,\eta) &= \frac{1}{1-\alpha} \mathbb{E}_P\left[h(W,Z)\mu_0(W,Z)\right] + \int_z p(z)\eta(z) (1 - \frac{1}{1-\alpha}\mathbb{E}_P[h(W,Z)\mid Z=z] ~dz\\
    &= \frac{1}{1-\alpha} \mathbb{E}_P\left[h(W,Z)\mu_0(W,Z)\right] + \mathbb{E}_P[\eta(Z)] - \frac{1}{1-\alpha}\mathbb{E}_P[h(W,Z)\eta(Z)]\\
    &= \frac{1}{1-\alpha} \mathbb{E}_P\left[h(W,Z)(\mu_0(W,Z)-\eta(Z))\right] + \mathbb{E}_P[\eta(Z)].
\end{align}

The Lagrange dual is then given by $\min_{\eta:\mathcal{Z}\to\mathbb{R}} \max_{h:\mathcal{Z}\times\mathcal{W}\to [0,1]} \mathcal{L}(h,\eta)$. To maximize $h$ out of this equation, observe that the optimal solution occurs when $h$ equals one whenever $\mu_0 - \eta$ is positive and zero when $\mu_0 - \eta$ is negative or $h(z,w) = \mathbb{I}(\mu_0(w,z) > \eta(z))$. Then observe that $\mathbb{I}(\mu_0(w,z) > \eta(z))(\mu_0(w,z)-\eta(z))$ can be rewritten as $(\mu_0(w,z)-\eta(z))_+$ where $(x)_+ = \max \{x,0\}$. Finally, because the original problem is a linear program, strong duality holds and we arrive at our final expression 

\begin{align}
    R_{\alpha,0} &= \min_{\eta:\mathcal{Z}\to\mathbb{R}} \frac{1}{1-\alpha} \mathbb{E}_P\left[(\mu_0(W,Z)-\eta(Z))_+\right] + \mathbb{E}_P[\eta(Z)]
\end{align}

\section{Proof of Theorem 1}
In this section, we provide a proof of Theorem \ref{thm:conv} in the main paper. This proof draws heavily on results from \citet{chernozhukov2018double} and \citet{jeong2020robust}. For notational simplicity and consistency between our work and theirs, let $\theta_0 = R_{\alpha,0}$ be the target parameter. Algorithm \ref{alg:alg1} in the main paper is an instance of the DML2 algorithm from \citet{chernozhukov2018double} where the score function $\psi$ is given by
\begin{align}
    \psi(O;\theta,\gamma) = \psi^b(O;\gamma) - \theta,
\end{align}
where
 \begin{align}
	\psi^b(O;\gamma) = \frac{1}{\alpha} (\mu(W,Z) - \eta(Z))_+ + \eta(Z) + \frac{1}{\alpha} h(W,Z) (\ell(Y,\mathcal{M}(X)) - \mu(W,Z)),
\end{align}
and where $O = (W,Z,V)$, $\gamma = (\mu,\eta)$, and $h = [\mu > \eta]$. In this proof, we will show that Assumptions 3.1 and 3.2 of \citet{chernozhukov2018double} are nearly satisfied and will fill in the gaps where they are not. We restate these assumptions here with some of the notation changed to match the notation used in this paper. First, some definitions: Let $c_0 > 0$, $c_1 > 0$, $s > 0$, $q > 2$ be some finite constants such that $c_0 \leq c$ and let $\{\delta_N\}_{N \geq 1}$ and $\{\Delta_N\}_{N \geq 1}$ be some positive constants converging to zero such that $\delta_N \geq N^{-1/2}$. Also, let $K \geq 2$ be some fixed integer, and let $\{\mathcal{P}_N\}_{N\geq 1}$ be some sequence of sets of probability distributions $P$ of $O$ on $\mathcal{O} = \mathcal{W}\times\mathcal{Z}\times\mathcal{V}$. Let $T$ be a convex subset of some normed vector space repressenting the set of possible nuissance parameters (i.e., $\gamma \in T$). Finally, let $a \lesssim b$ denote that there exists a constant $C$ such that $a \leq Cb$.
\begin{assumption} \label{ass:chern_31}
	(Assumption 3.1 from \citet{chernozhukov2018double}) For all $N \geq 3$ and $P \in \mathcal{P}_N$, the following conditions hold. (a) The true parameter value $\theta_0$ obeys $\mathbb{E}_{P}[\psi(O;\theta_0,\gamma_0)] = 0$. (b) The score $\psi$ can be written as $\psi(O;\theta,\gamma) = \psi^a(O;\gamma)\theta + \psi^b(O;\gamma)$. (c) The map $\gamma \mapsto \mathbb{E}_P[\psi(O;\theta,\gamma)]$ is twice continuously Gateaux-differentiable on $T$. (d) The score $\psi$ obeys Neyman orthogonality. (e) The singular values of the matrix $J_0 = \mathbb{E}_P[\psi^a(O;\gamma_0)]$ are between $c_0$ and $c_1$.
\end{assumption}

\begin{assumption} \label{ass:chern_32}
	(Assumption 3.2 from \citet{chernozhukov2018double}) For all $N \geq 3$ and $P \in \mathcal{P}_N$, the following conditions hold. (a) Given a random subset $I$ of $[N]$ of size $n = N/K$, the nuisance paramter estimator $\hat{\gamma} = \hat{\gamma}((O_i)_{i \in I^c})$ belongs to the realization set $\mathcal{T}_N$ with probability at least $1 - \Delta_N$, where $\mathcal{T}_N$ contains $\gamma_0$ and is constrained by the next conditions. (b) The following moment conditions hold:
	\begin{align*}
		m_N &= \sup_{\gamma \in \mathcal{T}_N} (\mathbb{E}_P[\|\psi(O;\theta_0,\gamma)\|^q])^{1/q} \leq c_1\\
		m_N' &= \sup_{\gamma \in \mathcal{T}_N} (\mathbb{E}_P[\|\psi^a(O;\gamma)\|^q])^{1/q} \leq c_1.
	\end{align*}
	(c) The following conditions on the statiestical rates $r_N$, $r_N'$, and $\lambda_N'$ hold:
	\begin{align*}
		r_N &= \sup_{\gamma \in \mathcal{T}_N} \|\mathbb{E}_P[\psi^a(O;\gamma)] - \mathbb{E}_P[\psi^a(O;\gamma_0)]\| \leq \delta_N,\\
		r'_N &= \sup_{\gamma \in \mathcal{T}_N} (\mathbb{E}_P[\|\psi(O;\theta, \gamma) - \mathbb{E}_P[\psi(O;\theta_0,\gamma_0)\|^2])^{1/2} \leq \delta_N,\\
		\lambda'_N &= \sup_{r\in(0,1),\gamma \in \mathcal{T}_N} \| \partial^2_r \mathbb{E}_P[\psi(O;\theta_0, \gamma_0 + r(\gamma - \gamma_0))]\| \leq \delta_N/\sqrt{N}.
	\end{align*}
	(d) The variance of the score $\psi$ is non-degenerate: All eigenvalues of the matix $\mathbb{E}_P[\psi(O;\theta_0,\gamma_0)\psi(O;\theta_0,\gamma_0)']$ are bounded from below by $c_0$.
\end{assumption}

Here, we will show that all of these conditions are satisfied except for Assumption \ref{ass:chern_31} (c) and, by extension the bound on $\lambda'_N$ in Assumption \ref{ass:chern_32}. These two conditions are used in \citet{chernozhukov2018double} to prove that, for any sequence $\{P_N\}_{N \geq 1}$ such that $P_N \in \mathcal{P}_N$, the following holds for all $P_N \in \mathcal{P}_N$
\begin{align}
	\|R_{N,2}\| = \mathcal{O}_{P_N}(\delta_N/\sqrt{N}),
\end{align}
where
\begin{align}
	R_{N,2} = \frac{1}{K} \sum_k \mathbb{E}_{n,k}[\psi(O;\theta_0,\hat{\gamma}_k)] - \frac{1}{N}\sum_{i=1}^N \psi(W_i;\theta_0,\gamma_0),
\end{align}
and where $\mathbb{E}_{n,k}[\cdot] = \frac{1}{n}\sum_{i \in I_k} (\cdot)$ is the empirical expectation w.r.t. the $k$'th cross-validation fold. We will prove this using other means. First, we establish that all other conditions in Assumptions \ref{ass:chern_31} and \ref{ass:chern_32} hold for all $P_N \in \mathcal{P}_N$. For notational simplicity, we will drop the dependence of $\ell$, $\gamma$, and $h$ on $O$ throughout. Additionally, denote by $\mathcal{E}_N$ the event that $\gamma_k \in \mathcal{T}_N$.

\paragraph{Proof of Assumption \ref{ass:chern_31} (a)} This holds trivially via the definitions of $\theta_0$ and $\mu_0$.

\begin{align}
	\mathbb{E}_{P_N}[\psi(O;\theta_0,\gamma_0)] &= \mathbb{E}_{P_N}\left[\frac{1}{\alpha} (\mu_0 - \eta_0)_+ + \eta_0 + \frac{1}{\alpha} h_0 (\ell - \mu_0) - \frac{1}{\alpha} (\mu_0 - \eta_0)_+ - \eta_0\right]\\
	&= \mathbb{E}_{P_N}\left[\frac{1}{\alpha} h_0 (\ell - \mu_0)\right] = \mathbb{E}_{P_N}\left[\frac{1}{\alpha} h_0 (\mu_0 - \mu_0)\right] = 0
\end{align}

\paragraph{Proof of Assumption \ref{ass:chern_31} (b)} This holds trivially with $\psi^a = -1$.

\paragraph{Proof of Assumption \ref{ass:chern_31} (d)} To show Neyman orthogonality of $\psi$, we must show that, for $P_N \in \mathcal{P}_N$, $T$ the set of possible nuissance parameter values, and $\tilde{T} = \{\gamma - \gamma_0:\gamma \in T\}$, the Gateaux derivative map $D_r: \tilde{T} \to \mathbb{R}$ exists for all $r \in [0,1)$ where
\begin{align*}
	D_r[\gamma-\gamma_0] = \partial_r \left\{ \mathbb{E}_{P_N} \left[\psi(O;\theta_0, \gamma_0 + r(\gamma - \gamma_0)) \right] \right\},\,\, \gamma \in T,
\end{align*}
and that $D_r[\gamma-\gamma_0]$ vanishes for $r = 0$. For notational simplicity, let $\mu_r  = \mu_0 - r(\mu - \mu_0)$, with analgous definitions for $\eta_r$ and $h_r$. Then, using Danskin's theorem, $D_r[\gamma-\gamma_0]$ exists for $r \in [0,1)$ and is given by
\begin{align*}
	D_r[\gamma-\gamma_0] = &\mathbb{E}_{P_N}\left[ \frac{1}{\alpha} [\mu_r \geq \eta_r]((\mu - \mu_0) - (\eta - \eta_0)) + (\eta - \eta_0)\right.\\ &\left.+ \frac{1}{\alpha}(h - h_0)(l - \mu_r) - \frac{1}{\alpha}h_r(\mu - \mu_0) \right]
\end{align*}
Finally, we have
\begin{align*}
	\left. D_r[\gamma-\gamma_0] \right|_{r=0} &= \mathbb{E}_{P_N}\left[ \frac{1}{\alpha} [\mu_0 \geq \eta_0]((\mu - \mu_0) - (\eta - \eta_0)) + (\eta - \eta_0) + \frac{1}{\alpha}(h - h_0)(l - \mu_0) - \frac{1}{\alpha}(h_0)(\mu - \mu_0) \right]\\
	&= \mathbb{E}_{P_N}[ \eta - \eta_0] - \frac{1}{\alpha} \mathbb{E}_{P_N}[h_0(\eta - \eta_0)]\\
	&= \mathbb{E}_{P_N}[ \eta - \eta_0] - \frac{1}{\alpha} \mathbb{E}_{P_N}[\mathbb{E}_{P_N}[h_0 \mid Z] (\eta - \eta_0)]\\
	&= \mathbb{E}_{P_N}[ \eta - \eta_0] - \frac{1}{\alpha} \mathbb{E}_{P_N}[\alpha (\eta - \eta_0)] = 0
\end{align*}

The second line follows from the definitions of $h_0 = [\mu_0 \geq \eta_0]$ and $\mu_0 = \mathbb{E}[\ell \mid W,Z]$. The final line follows from the constraint that $\mathbb{E}_{P_N}[h_0 \mid Z] = \alpha$ almost everywhere.

\paragraph{Proof of Assumption \ref{ass:chern_31} (e)} This hold trivially since $J_0 = \mathbb{E}_{P_N}[\psi^a(O;\gamma_0)] = -1$.

\paragraph{Proof of Assumption \ref{ass:chern_32} (a)} This holds by construction of $\mathcal{T}_N$ and Assumption \ref{ass:conv_rate}.

\paragraph{Proof of Assumption \ref{ass:chern_32} (b)} The bound on $m'_N$ holds trivially as $\psi^a(O;\gamma) = -1$. To bound $m_N$ on the event $\mathcal{E}_N$, we begin by decomposing it using the triangle inequality as
\begin{align}
    \|\psi(O;\theta_0,\gamma)\|_{P_N,q} &= \left\|\frac{1}{\alpha}(\mu - \eta)_+ + \eta + \frac{1}{\alpha}h(\ell - \mu)\right\|_{P_N,q}\\
    &\leq \frac{1}{\alpha}\|(\mu - \eta)_+ + \alpha\eta\|_{P_N,q} + \frac{1}{\alpha}\|h(\ell - \mu)\|_{P_N,q}
\end{align}

Since $0 \leq h \leq 1$, and by the triangle inequality and Assumption \ref{ass:conv_rate}, we have 
\begin{align}
\|h(\ell - \mu)\|_{P_N,q} &\leq \|\ell - \mu\|_{P_N,q}\\
&=  \|\ell - \mu_0 + \mu_0 - \mu\|_{P_N,q}\\
&\leq \|\ell - \mu_0\|_{P_N,q} + \|\mu_0 - \mu\|_{P_N,q}\\
&\leq 2C,
\end{align}
where the fourth line follows from Assumption \ref{ass:conv_rate} (a). Next, we can bound $\|(\mu - \eta)_+ + \alpha\eta\|_{P_N,q}$ as
\begin{align}
	    \|(\mu - \eta)_+ + \alpha\eta\|_{P_N,q} &\leq \|(\mu - \eta)_+\|_{P_N,q} + \|\alpha\eta\|_{P_N,q}\\
		&\leq \|\mu - \eta\|_{P_N,q} + \|\alpha\eta\|_{P_N,q}\\
		&\leq \|\mu_0\|_{P_N,q} + \|\mu_0 - \mu\|_{P_N,q} + (1+\alpha)(\|\eta_0\|_{P_N,q} + \|\eta_0 - \eta\|_{P_N,q})\\
		&\leq (4 + 2\alpha)C
\end{align}
where the fourth line follows from Jensen's inequality and Assumption \ref{ass:conv_rate} (a). Thus, $\|\psi(O;\theta_0,\gamma)\|_{P_N,q} < \infty$ and Assumption \ref{ass:chern_32} (b) holds.

\paragraph{Proof of Assumption \ref{ass:chern_32} (c)} The bound on $r_N$ is trivially satisfied. Further,
\begin{align}
	\|\psi(O;\theta_0,\gamma) - \psi(O;\theta_0,\gamma_0)\|_{P_N,2} &= \frac{1}{\alpha}\|\alpha(\eta - \eta_0) + \ell(h - h_0) + h_0\eta_0 - h\eta\|_{P_N,2}\\
	&\leq \frac{1}{\alpha}\left(\alpha\|\eta - \eta_0\|_{P_N,2} + \|\ell(h - h_0)\|_{P_N,2} + \|h_0\eta_0 - h\eta\|_{P_N,2}\right)\\
	&\leq \frac{1}{\alpha}(\alpha \delta_N + C\delta_N + \|h_0\eta_0 - h\eta\|_{P_N,2})
\end{align}
where the first line follows from the definition of $\psi$ and $h$, the second line follows from the triangle inequality, and the third line follows from the Assumption \ref{ass:conv_rate}. Then, to bound $\|h_0\eta_0 - h\eta\|_{P_N,2}$, first observe that $\|h_0 - h\|_{P_N,2} = \|[h_0 = 1,h=0] + [h_0 = 0,h=1]\|_{P_N,2}$ where $[\cdot]$ is the Iverson bracket. Then, we have
\begin{align}
	\|h_0\eta_0 - h\eta\|_{P_N,2} &= \|[h_0=1,h=0]\eta_0 - [h_0=0,h=1]\eta + [h_0 = 1,h = 0](\eta_0 - \eta)\|_{P_N,2}\\
	&\leq \|([h_0=1,h=0] - [h_0=0,h=1])\eta_0 - [h_0=0,h=1](\eta - \eta_0) + [h_0 = 1,h = 0](\eta_0 - \eta)\|_{P_N,2}\\
	&\leq \|([h_0=1,h=0] - [h_0=0,h=1])\eta_0\|_{P_N,2} + \| [h_0=0,h=1](\eta - \eta_0)\|_{P_N,2}\\ &\quad+ \|[h_0 = 1,h = 0](\eta_0 - \eta)\|_{P_N,2}\\
	&\leq C\|h_0 - h\|_{P_N,2} + C \delta_N + \delta_N\\
	&\lesssim \delta_N,
\end{align}
where the third line follows from the triangle inequality and the fourth line follows from Lemma 14 of \cite{jeong2020robust} and Assumption \ref{ass:conv_rate} (d) - (f). Finally, we have 
\begin{align}
	\|\psi(O;\theta_0,\gamma) - \psi(O;\theta_0,\gamma_0)\|_{P_N,2} \leq \frac{1}{\alpha}(\alpha + 3C + 1)\delta_N
\end{align}
and thus the bound on $r'_N$. As discussed above the bound on $\lambda'_N = \sup_{r\in(0,1),\gamma \in \mathcal{T}_N} \| \partial^2_r \mathbb{E}_{P_N}[\psi(O;\theta_0, \gamma_0 + r(\gamma - \gamma_0))]\|$ does \textit{not} hold because $\psi$ is not twice differentiable.

\paragraph{Proof of Equation A.6 from \cite{chernozhukov2018double}} We have now shown that all parts of Assumptions \ref{ass:chern_31} and \ref{ass:chern_32} hold except for assumptions involving the second derivative of $\psi$. The assumptions regarding the second derivative of $\psi$ are used by \cite{chernozhukov2018double} to show that, for all $P_N \in \mathcal{P}_N$
\begin{align}
	R_{N,2} = \frac{1}{K} \sum_k \mathbb{E}_{n,k}[\psi(O;\theta_0,\hat{\gamma}_k)] - \frac{1}{N}\sum_{i=1}^N \psi(O_i;\theta_0,\gamma_0) = \mathcal{O}_{P_N}(\delta_N/\sqrt{N}).
\end{align}

We will show this using similar arguments to those in Step 3 of the proof of Theorem 3.1 in \cite{chernozhukov2018double}, but without relying on the second derivative of $\psi$. As in \citet{chernozhukov2018double}, because the number of cross-validation folds $K$ is a fixed integer, we need only show that
\begin{align}
    \mathbb{E}_{n,k}[\psi(O;\theta_0,\hat{\gamma}_k)] - \frac{1}{n}\sum_{i=1}^n \psi(O_i;\theta_0,\gamma_0) = \mathcal{O}_{P_N}(\delta_N/\sqrt{N})
\end{align}
where $n = N/K$. Following \cite{chernozhukov2018double} this quantity can be bounded as
\begin{align}
    \left\|\mathbb{E}_{n,k}[\psi(O;\theta_0,\hat{\gamma}_k)] - \frac{1}{n}\sum_{i=1}^n \psi(O_i;\theta_0,\gamma_0)\right\| \leq \frac{\mathcal{I}_{3,k} + \mathcal{I}_{4,k}}{\sqrt{n}},
\end{align}
where
\begin{align}
    \mathcal{I}_{3,k} &= \|\mathbb{G}_{n,k}[\psi(O;\theta_0,\hat{\gamma}_k)] - \mathbb{G}_{n,k}[\psi(O;\theta_0,\gamma_0)]\|\\
    \mathcal{I}_{4,k} &= \sqrt{n}\|\mathbb{E}_{P_N}[\psi(O;\theta_0,\hat{\gamma}_k)] \mid (O_i)_{i\in I_k} - \mathbb{E}_{P_N}[\psi(O;\theta_0,\gamma_0)]\|,
\end{align}
and where
\begin{align}
    \mathbb{G}_{n,k}[\phi(O)] = \frac{1}{\sqrt{n}}\sum_{i\in I_k}\left(\phi(O_i) - \int \phi(w) dP_N\right).
\end{align}

\citet{chernozhukov2018double} showed that $\mathcal{I}_{3,k} = \mathcal{O}_{P_N}(r'_N)$ (using only assumptions satisfied by Assumption \ref{ass:conv_rate}) and thus, what remains to be shown is that $\mathcal{I}_{4,k} \leq \delta_N/\sqrt{N}$ which we do drawing on proofs in \cite{jeong2020robust}. Define
\begin{align}
    f_k(r) = \mathbb{E}_{P_N}[\psi(O;\theta_0,\gamma_0 + r(\hat{\gamma}_k - \gamma_0)) \mid (O_i)_{i\in I^c_k}] - \mathbb{E}_{P_N}[\psi(O;\theta_0,\gamma_0)].
\end{align}

Note that $|f_k(1)|$ is the quantity that we want to bound and $f_k(0) = 0$. Then, using the mean value theorem, for some $r^* \in (0,1)$, we have $|f_k(1)| = |f_k(0) + f'_k(r^*)|  = |f'_k(r^*)| \leq \sup_r |f'_k(r)|$. Define $\hat{\mu}_r = \mu_0 + r(\hat{\mu}_k - \mu_0)$ with analogous definitions for $\hat{\eta}_r$ and $\hat{h}_r$. Then, for arbitrary $r \in (0,1)$, we bound $|f'_k(r)|$ as
\begin{align}
    |f'_k(r)| &= |\partial_r \mathbb{E}_{P_N}[\psi(O;\theta_0,\gamma_0 + r(\hat{\gamma}_k - \gamma_0)) \mid (O_i)_{i\in I^c_k}]|\\
    &\leq \frac{1}{\alpha}\left(\left|\mathbb{E}_{P_N}[([\hat{\mu}_r > \hat{\eta}_r] - h_0)(\hat{\mu}_k - \mu_0)]\right| + \left|\mathbb{E}_{P_N}[(\alpha - [\hat{\mu}_r > \hat{\eta}_r])(\hat{\eta}_k - \eta_0)]\right| + 2\left|\mathbb{E}_{P_N}[(\hat{h}_k - h_0)(\hat{\mu}_k - \mu_0)]\right| \right)\\
    &\leq \frac{1}{\alpha}\left(\left|\mathbb{E}_{P_N}[(\alpha - [\hat{\mu}_r > \hat{\eta}_r])(\hat{\eta}_k - \eta_0)]\right| + 3\left|\mathbb{E}_{P_N}[(\hat{h}_k - h_0)(\hat{\mu}_k - \mu_0)]\right| \right)\\
    &\leq \frac{1}{\alpha}\left(\|\hat{h}_k - h_0\|_{P_N,1}\|\hat{\eta}_k - \eta_0\|_{P_N,\infty} + 3\|\hat{h}_k - h_0\|_{P_N,1}\|\hat{\mu}_k - \mu_0\|_{P_N,\infty}\right)\\
    &\leq \frac{1}{\alpha} \|\hat{h}_k - h_0\|_{P_N,1} \left( \|\hat{\eta}_k - \eta_0\|_{P_N,\infty} + 3\|\hat{\mu}_k - \mu_0\|_{P_N,\infty}\right)
\end{align}
where the second line follows from the triangle inequality, the third line follows from the obsrevation that $|[\hat{\mu}_r < \hat{\eta}_r] - h_0| \leq |\hat{h}_k - h_0|$, and the fourth line follows from this observation and applications of Jensen's inequality followed by H{\"o}lder's inequality. By Lemmas 13 and 14 of \citet{jeong2020robust} and Assumption \ref{ass:conv_rate} (f), we have
$\|\hat{\eta} - \eta_0\|_{P_N,\infty} = \mathcal{O}(\|\hat{\mu}_k - \mu_0\|_{P_N,\infty})$ and
$\|\hat{h}_k - h_0\|_{P_N,1} = \mathcal{O}(\delta_N N^{-1/6})$. Further, by Assumption \ref{ass:conv_rate} we have $\|\hat{\mu}_k - \mu_0\|_{P_N,\infty} =  \mathcal{O}(\delta_N N^{-1/3})$. Thus, we have $|f'_k(r)| = \mathcal{O}(\delta_N N^{-1/2})$ and our proof is concluded.


%
%
%
%
%

\section{Handling discrete W}
\label{app:disc_w}

\begin{algorithm}[h!]
    \SetAlgoLined
    \KwIn{Model $\mathcal{M}$, Dataset $\mathcal{D} = \{(w_i,z_i,v_i)\}_{i=1}^n$, noise parameter $\epsilon$, and $K$ cross-validation folds $I_k\subset \{1,\dots,n\}$ and $I^c_k = \{1,\dots,n\} \setminus I_k$}
    \For{$k=1,\dots,K$}{
        Estimate $\hat{\mu}_k \approx \mu_0$ using data in $I^c_k$\\
        Estimate $\hat{\eta}_k \approx \eta_0$ according to Eq. \ref{eq:eta0} using $\hat{\mu}_k + u_i$ and data in $I^c_k$\\
		\For{$i \in I_k$}{
			Let $\hat{\mu}_{i} = \hat{\mu}_k(w_i,z_i)$\\
			Let $\hat{\eta}_{i} = \hat{\eta}_k(z_i)$\\
			Let $\hat{h}_{i} = [\hat{\mu}_{i} + u_i > \hat{\eta}_{i}]$\\
		}
    }
	
    Let $\hat{R}_{\alpha,\epsilon} = \frac{1}{K} \sum_{k} \frac{1}{|I_k|} \sum_{i \in I_k} \frac{1}{1-\alpha}(\hat{\mu}_{i} + u_i - \hat{\eta}_{i})_+ + \hat{\eta}_{i}$\\
    \qquad\qquad\qquad\qquad$+ \frac{1}{1-\alpha}\hat{h}_{i}(\ell(y_i,\mathcal{M}(x_i)) - \hat{\mu}_{i})$\\
    \KwResult{$\hat{R}_\alpha$}
    \caption{\textsc{Discrete Worst-case Sampler}}
    \label{alg:alg2}
\end{algorithm}

When $W$ contains only discrete variables, Assumption \ref{ass:conv_rate} (f) no longer holds. In such cases, we can retain the desirable theoretical properties of Theorem \ref{thm:conv} at the cost of an arbitrarily small, user-controlled amount of bias, using a simple augmentation to the \textsc{Worst-case Sampler} in Algorithm \ref{alg:alg1} of the main paper. This augmentation works by adding a small amount of user-controlled noise to the $\mu_0$ thereby smoothing the conditional distribution of $\mu_0$ given $Z$. To derive this augmentation, first let $U \sim Unif(0,\epsilon)$ be a uniform random variable with support on $[0,\epsilon]$ such that $U \indep \{W,Z,V\}$. Then, we can choose $h$ to maximize the expected loss plus this extra noise term as

\begin{align}
    R_{\alpha,\epsilon,0} = &\sup_{h:[0,\epsilon]\times\mathcal{W} \times \mathcal{Z}\rightarrow [0,1]} \,\, \frac{1}{1-\alpha} \mathbb{E}_P\left[h(U,W,Z)\mu_{\epsilon,0}(U,W,Z)\right]\\
    &\text{s.t.} \quad \,\,\mathbb{E}_P[h(U,W,Z)\mid Z] = 1-\alpha\quad a.e.,
\end{align}

where $\mu_{\epsilon,0}(U,W,Z) = \mu_0(W,Z) + U$. The corresponding change to the estimation algorithm is shown in Algorithm \ref{alg:alg2}. This algorithm returns a consistent estimate for $R_{\alpha,\epsilon,0}$ as $\mu_{\epsilon,0}$ satisfies the conditions of Assumption \ref{ass:conv_rate} (f). In the following proposition, we show that the difference between $R_{\alpha,0}$ and $R_{\alpha,\epsilon,0}$ is bounded by $\epsilon$.

\begin{proposition}
$|R_{\alpha,\epsilon,0} - R_{\alpha,0}| \leq \epsilon$
\end{proposition}

\begin{proof}
First, since $U$ has support in the non-negatives, we have $R_{\alpha,\epsilon,0} \geq R_{\alpha,0}$. Next, let $S = \{\mathcal{W}\times\mathcal{Z} \to [0,1] : \mathbb{E}_P[h\mid Z] = 1-\alpha ~~a.e.\}$ and $\tilde{S} = \{[0,\epsilon]\times\mathcal{W}\times\mathcal{Z} \to [0,1] : \mathbb{E}_P[h\mid Z] = 1-\alpha ~~a.e.\}$. Then,
\begin{align}
    R_{\alpha,\epsilon,0} &= \max_{\tilde{h}\in \tilde{S}} \frac{1}{1-\alpha}\mathbb{E}_P[\tilde{h}(\mu_0 + U)]\\
    &\leq \max_{\tilde{h}\in \tilde{S}} \frac{1}{1-\alpha}\mathbb{E}_P[\tilde{h}(\mu_0 + \epsilon)]\\
    &= \left(\max_{\tilde{h}\in \tilde{S}} \frac{1}{1-\alpha}\mathbb{E}_P[\tilde{h}\mu_0]\right) + \epsilon\\
    &= \left(\max_{h\in S} \frac{1}{1-\alpha}\mathbb{E}_P[h\mu_0]\right) + \epsilon\\
    &= R_{\alpha,0} + \epsilon.
\end{align}
\end{proof}

\section{Experimental Details}
\subsection{Dataset}
We loosely follow the setup of \citet{giannini2019machine} in deriving the dataset for training sepsis diagnosis models. The dataset contains electronic health record data collected over four years at our institution's hospital (Hospital A). The dataset consists of 278,947 emergency department patient encounters. The prevalence of the target disease, sepsis (S), is $2.1\%$. 17 features pertaining to vital signs (V) (heart rate, respiratory rate, temperature), lab tests (L) (white blood cell count [wbc], lactate), and demographics (D) (age, gender) were extracted. For encounters that resulted in sepsis (i.e., positive encounters), physiologic data available up until sepsis onset time was used. For non-sepsis encounters, all data available until discharge was used. For each of the time-series physiologic features (V and L), min, max, and median summary features were derived. Unlike vitals, lab measurements are not always ordered (O) and are subject to missingness (lactate $89\%$, wbc $27\%$). To model lab missingness, missingness indicators (O) for the lab features were added. The evaluation dataset was created using a held-out sample of 10,000 patients. The remaining data was used to train the two models. As per its definition, qSOFA was computed from respiratory rate, systolic blood pressure, and glasgow coma score (gcs) (gcs and blood pressure were separately extracted for these patients) \citep{singer2016third}. Using existing standards, we remapped gcs to the Alert, Voice, Pain, Unresponsive (AVPU) score which is required to compute qSOFA \citep{gardner2006value}.

\begin{figure}[!ht]
 \centering
	\includegraphics[scale=0.33]{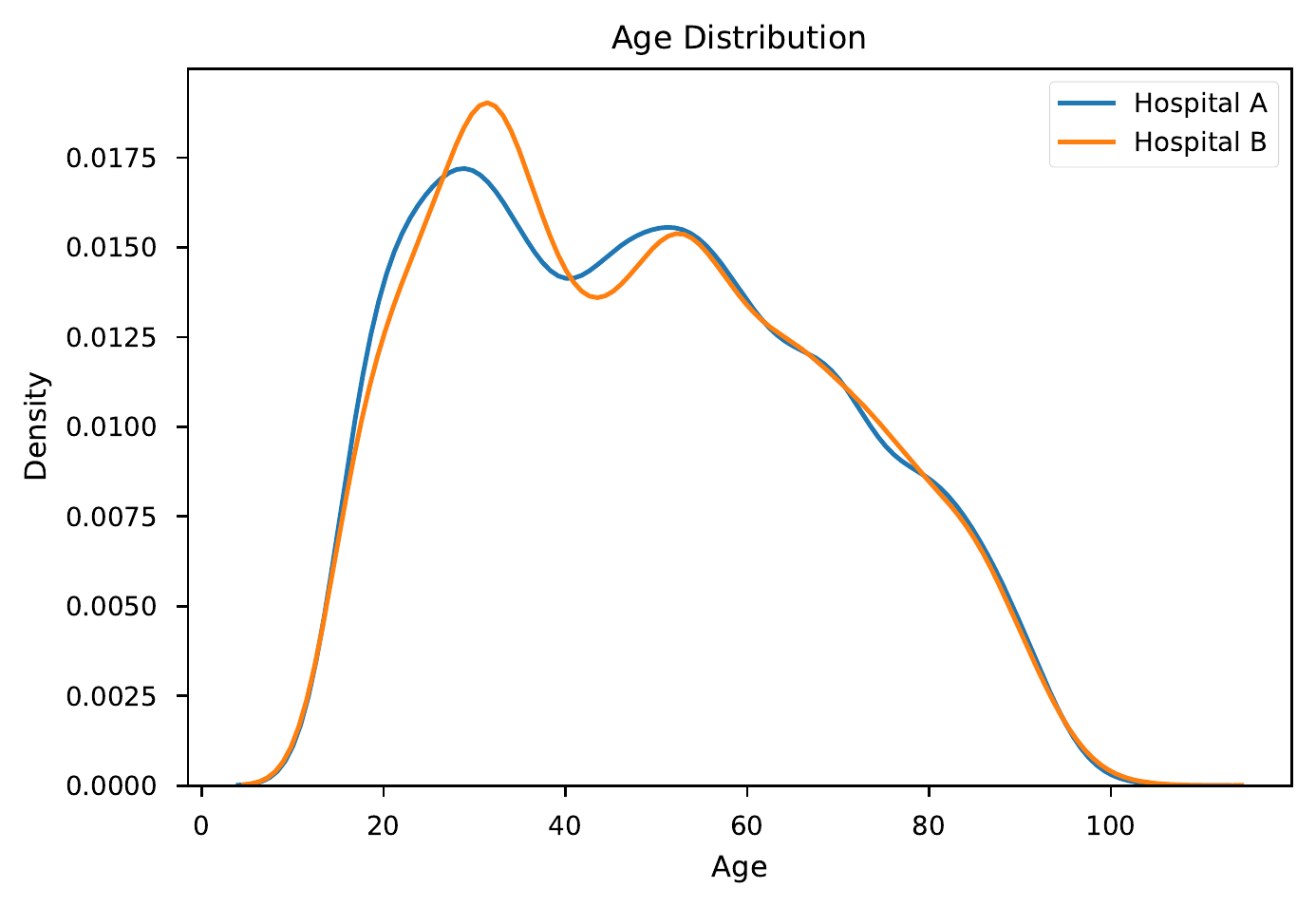}
  \caption{Age distributions at the two hospitals are very similar.}
  \label{fig:two-hosp-age}
\end{figure}

In Section 3.2.3 we use data from another Hospital B (within the same network as Hospital A used in the prior experiments). This dataset contains 96,574 patient encounters and has a sepsis prevalence of $2.8\%$ (vs $2.1\%$ at Hospital A). Turning to demographics, the population at Hospital A is $42\%$ female while at Hospital B it is $39\%$ female. Finally, Kernel Density Estimates (Fig \ref{fig:two-hosp-age}) of the age distributions at the two hospitals were very similar.  The missingness rates for lab orders were $28\%$ wbc missingness (unchanged from Hospital A) and $77\%$ lactate missingness ($12\%$ decrease in missingness). Thus, there is a sizeable increase in (lactate) test ordering patterns from Hospital A to Hospital B.

\subsection{Models}
\begin{figure}[!ht]
\centering
\begin{tikzpicture}[scale=0.8, transform shape]
          \def\unit{1.25}
            
            \node (d) at (0,0.25) [label=above:D,point];
            \node (s) at (-.75*\unit, -.5*\unit) [label=above:S,point]; 
            \node (o) at (0.75*\unit, -.5*\unit) [label=above:O,point];
            \node (v)at (-0.5*\unit, -1*\unit) [label=below:V,point];
            \node (l) at (0.5*\unit, -1*\unit) [label=below:L,point];
            
            \path (d) edge (s);
            \path (d) edge (o);
            \path (s) edge (o);
            \path (s) edge (v);
            \path (s) edge (l);
            \path (o) edge (l);
            \path (d) edge (v);
            \path (d) edge (l);

\end{tikzpicture}
\caption{Posited DAG of the data generating process used to train the robust model. The robust model was trained to be stable to shifts in the policy for ordering lab tests (O).}
\label{fig:sepsis-dag}
\end{figure}
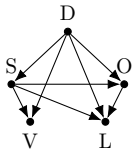

The \textbf{classical} model was a Random Forest Classifier trained using the implementation in scikit-learn \citep{pedregosa2011scikit}. The hyperparameters were tuned via grid search 5-fold cross-validation (CV) on the training dataset. The resulting tuned parameters were: number of trees 1000, min samples in a leaf 1, max depth 15, and min samples per split 2.

The \textbf{robust} model was an implementation of the surgery estimator \citep{subbaswamy2019preventing}, a causal method for training models which make predictions that are stable to pre-specified shifts. Using the DAG in Fig \ref{fig:sepsis-dag}, lab test ordering patterns are defined by the distribution $P(O|S,D)$. As opposed to the classical model, which models $P(S|V,D,O,L)$, the robust model models the \emph{interventional} distribution $P(S|V,D,do(O), L)$ which considers a hypothetical intervention on test ordering patterns. The robust model was fit by inverse probability weighting (IPW). First, we fit a logistic regression model of $P(O|S,D)$ using the training data. Then, a Random Forest model with the full feature set was trained using sample weights $\frac{1}{P(o_i|s_i,d_i)}$. These weights create a training dataset in which lab orders are approximately randomized w.r.t. $S$ and $D$. The tuned hyperparameters from the classical model were used as the hyperparameters for the robust model.

\subsection{Estimating Worst-Case Risk}
Estimating the Worst-Case risk using Algorithm \ref{alg:alg1} requires specifying the variable sets $W$ and $Z$ and fitting the models for the nuisance parameters $\mu_0$ (the conditional expected loss) and $\eta_0$ (the conditional quantile of $\mu_0$). To measure model robustness to changes in test ordering patterns we define $W=\{O\}$ and $Z=\{S,D\}$ (so that $P(W|Z)$ corresponds to test ordering policies). Since classification accuracy was the performance metric of interest, we chose $0-1$ loss as the loss function. Within Algorithm \ref{alg:alg1}, 10-fold CV was used (i.e., $K = 10$).

To fit the conditional expected loss $\hat{\mu}_k$, we used the scikit-learn Kernel Ridge Regression implementation with the RBF kernel which minimizes $\ell_2$ regularized mean squared error (MSE). Because $W$ contained all discrete variables, we applied Algorithm \ref{alg:alg2} with noise $U \sim Unif(0, 1 \times 10^{-5})$. The bandwidth and regularization parameters were tuned using a nested 5-fold CV on each estimation fold $k$ to which Algorithm \ref{alg:alg2} was applied. As $\hat{\mu}_k$ does not depend on $\alpha$, it was not refit for different $\alpha$ values. 

To model the conditional quantile function $\hat{\eta}_k$, we used $\ell_2$ regularized quantile regression with a b-spline basis expansion. We used a quantile b-spline basis expansion for Age and added an interaction term between $S$ and all other variables in $Z$ (Age expansion and Gender). The regularization constant $\lambda_k$ was chosen separately for each estimation fold $k$ using a nested 5-fold cross-validation and grid search to produce the lowest mean absolute deviation.

\end{document}